\documentclass[10pt]{article} % For LaTeX2e

\usepackage[accepted]{tmlr}

\usepackage{graphicx}           % For including images
\usepackage{subcaption}         % Allows for the use of subfigures and subcaptions
\usepackage{url}                % For displaying URLs
\usepackage{titletoc}

\let\tmlrAND\AND
\let\AND\relax
\usepackage{amssymb}
\usepackage{amsthm}
\usepackage{mathtools}
\usepackage{bbm}
\usepackage{algorithm}
\usepackage{algorithmic}
\usepackage{array}
\usepackage{hyperref}
\theoremstyle{definition}
\newtheorem{proposition}{Proposition}
\newtheorem{lemma}{Lemma}
\newtheorem{corollary}{Corollary}
\newtheorem{assumption}{Assumption}

\newcommand{\R}{\mathbb{R}}
\newcommand{\E}[2][]{\mathbb{E}_{#1} \left[ {#2} \right]}
\newcommand{\indic}[1]{\mathbbm{1}\left[{#1}\right]}
\renewcommand{\epsilon}{\varepsilon}
\newcommand{\cA}{\mathcal{A}}
\newcommand{\cH}{\mathcal{H}}
\newcommand{\cN}{\mathcal{N}}
\newcommand{\iid}{\textnormal{i.i.d.}}

\usepackage{xcolor}

\newcommand{\RightComment}[1]{\hfill\(\triangleright\)~#1}

\usepackage{placeins}

\definecolor{softblue}{RGB}{70,130,220}
\hypersetup{
    colorlinks=true,
    citecolor=softblue,
    linkcolor=softblue,
    urlcolor=softblue
}

\def\month{11}
\def\year{2026}
\def\openreview{\url{https://openreview.net/forum?id=qmCljotj08}}

\title{Retry Policy Gradients in Continuous Action Spaces}

\author{\name Soichiro Nishimori \email nishimori@ms.k.u-tokyo.ac.jp \\
\addr The University of Tokyo, Japan
\tmlrAND
\name Paavo Parmas \email paavo.parmas@weblab.t.u-tokyo.ac.jp \\
\addr The University of Tokyo, Japan
}

\begin{document}

\maketitle

\begin{abstract}
     Retry-based objectives such as pass@K and max@K
optimize the best return obtained from multiple sampled trajectories,
and recent work has shown that they can promote exploration without
explicit exploration bonuses.  In discrete action spaces, ReMax
was shown to do so by adapting to return uncertainty.  In this work, we
apply pathwise derivative estimators for retry objectives and use
them to extend ReMax to continuous action spaces.  We study
the resulting learning dynamics in a stylized fixed-state setting and
show that, even with deterministic rewards, ReMax can promote higher
policy stochasticity by reshaping the policy-gradient landscape.
For a retry budget $M>1$, when the policy mean is not at a stationary
point of the cost, the ReMax gradient increases the policy scale once
the scale is sufficiently small.  We also derive
gradient bounds controlled by the distance of the best sampled action
to the optimum: they vanish at the deterministic optimum, and the
mean-gradient bound gives an upper envelope that decreases with the
retry budget.  We further show that Adam's adaptive normalization can
mitigate the resulting gradient damping, depending on its
numerical stabilization parameter.
Empirically, we instantiate this objective as \textbf{ReMax
Actor-Critic (ReMAC)}, an off-policy actor--critic algorithm that
optimizes the ReMax objective using a pathwise derivative estimator.
Our experiments show that ReMAC can promote higher policy entropy with $M>1$ than with $M=1$,
without entropy regularization, and achieves performance comparable
to SAC.
Code is available at \url{https://github.com/nissymori/remax_continuous}.
\end{abstract}

\section{Introduction}
Exploration remains a central challenge in reinforcement learning (RL) \citep{sutton1998intro}.
In continuous control, exploration is commonly encouraged by stochastic policies together with entropy regularization, as popularized by Soft Actor-Critic (SAC) \citep{haarnoja2018sac,christodoulou2019soft}.
Entropy regularization promotes stochasticity by explicitly modifying the reward-maximization objective, and therefore generally changes the corresponding optimal policy \citep{geist2019theory}.
While this modification can improve the optimization landscape and facilitate policy optimization \citep{ahmed2019understanding}, it introduces a trade-off between reward maximization and policy entropy that must be controlled through the regularization strength.

Retry-based objectives take a different route to exploration: rather than adding a bonus to the reward, they promote exploration by changing how sampled returns are aggregated, while leaving the optimal deterministic policy unchanged (Sec.~\ref{sec:remax:remax_objective}).
Such objectives maximize the highest return achieved over multiple ($M \in \mathbb{N}$) sampled trajectories, and have recently gained attention as a compelling formulation for RL in discrete action spaces \citep{koyamada2023emergence,nishimori2026emergence,tang2025optimizing,walder2025pass}.
In large language model (LLM) post-training, directly optimizing pass@K, which evaluates the success of the best response among $K$ generations, significantly improves output diversity \citep{tang2025optimizing,chen2025pass,walder2025pass,bagirov2025bonworlds}.
Within episodic RL settings, the ReMax objective was introduced to explicitly account for the uncertainty of returns \citep{koyamada2023emergence,nishimori2026emergence}.
Prior studies have revealed that ReMax naturally encourages exploration as an adaptive response to return uncertainty \citep{koyamada2023emergence,nishimori2026emergence}.
Furthermore, \citet{nishimori2026emergence} demonstrated that ReMax can maintain high policy entropy by slowing down convergence, even in environments with deterministic rewards.
However, these existing analyses are largely restricted to discrete action spaces, leaving the detailed theoretical properties of the ReMax objective incompletely understood.

In this study, we investigate the fundamental properties of the ReMax objective in continuous action spaces.
Specifically, we identify intrinsic characteristics of the deterministic ReMax policy gradient, in both its \textit{direction} and its \textit{magnitude}, that can sustain policy stochasticity when the retry budget satisfies $M>1$.
In a stylized fixed-state setting, we show that, for a sufficiently small policy scale, the ReMax gradient increases that scale whenever the policy mean is non-stationary.
We also derive gradient bounds controlled by the distance of the best sampled action to the optimum: they vanish at the deterministic optimum, and the mean-gradient bound gives an upper envelope that decreases with $M$.

Fig.~\ref{fig:vector_field} illustrates this effect.
It shows the vector field of ReMax gradients for a 1D Gaussian policy $a \sim \cN(\mu,\sigma^2)$ with reward $r(a)=-a^2$.
For larger $M$, when $\mu$ is far from the optimum and $\sigma$ is small (bottom edges), the gradient points toward increasing $\sigma$, whereas for $M=1$ (standard RL) the gradient decreases $\sigma$ in this example, as it does whenever the cost is strongly convex (Proposition~\ref{prop:variance_decrease}).
Conversely, when $\mu$ is near the optimal mean and $\sigma$ is large (top center), the gradient norm is attenuated (lighter color), slowing convergence.
Unlike entropy regularization, ReMax does not change the optimal parameters $\mu=0$ and $\sigma=0$, whereas entropy regularization shifts the optimum to $\mu=0$ and $\sigma>0$.
These properties can yield an optimization trajectory in which policy entropy first increases and then gradually decreases before approaching the deterministic optimal policy; with adaptive optimizers such as Adam, the second phase depends on the normalization (Sec.~\ref{sec:grad_analysis:adam_effect}).
A detailed explanation of Fig.~\ref{fig:vector_field} is provided in Sec.~\ref{sec:grad_analysis:vector_field}, and a theoretical explanation for why these properties emerge is given in Sec.~\ref{sec:grad_analysis:theoretical_explanation}.
Finally, the damping effect can be mitigated by Adam’s adaptive gradient normalization \citep{kingma2014adam}, the de facto standard optimizer in deep learning.
\begin{figure}[t]
    \centering
    \includegraphics[width=1.0\textwidth]{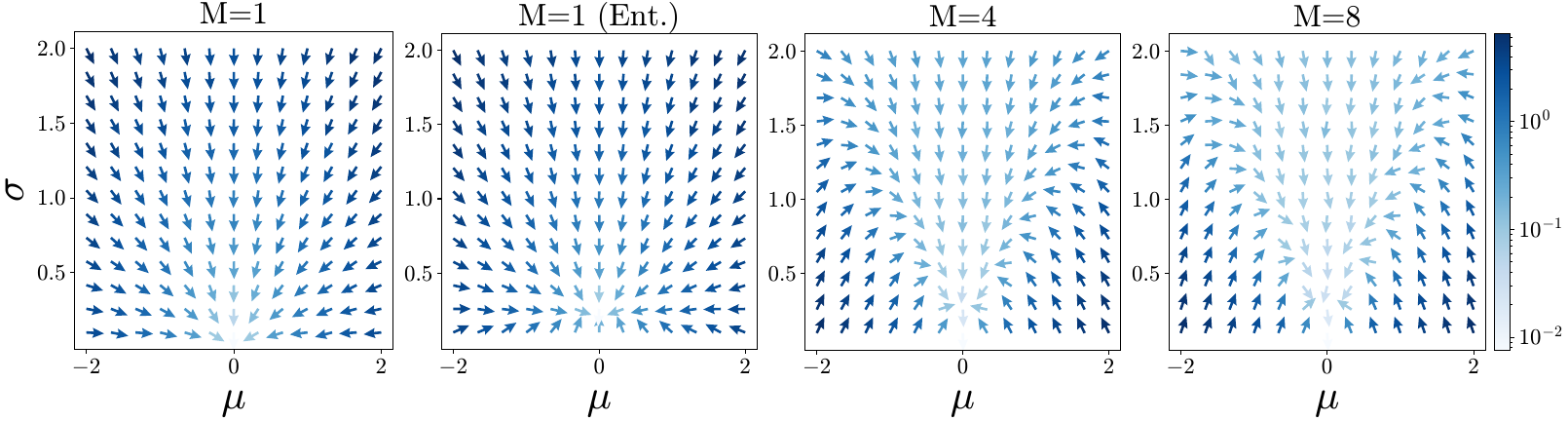}
    \caption{
        Vector field of normalized ReMax gradients for $M = 1, 4, 8$ over $\mu \in [-2,2]$ and $\sigma \in [0.1,2]$.
        We also plot the entropy-regularized standard RL ($M=1$ (Ent.)).
        Arrow color indicates gradient norm (darker is larger).
        For larger $M$, the gradient increases $\sigma$ when $\mu$ is far from the optimum (bottom edge) and the gradient norm near the optimal mean (small $\mu$, nonzero $\sigma$) shrinks, slowing convergence.
        Unlike entropy regularization, ReMax keeps the deterministic optimum $\mu=0,\sigma=0$ as its unique maximizer.
    }
    \label{fig:vector_field}
\end{figure}

We then instantiate the ReMax objective in an off-policy actor--critic algorithm, which we call \textbf{ReM}ax \textbf{A}ctor-\textbf{C}ritic (\textbf{ReMAC}), with a pathwise derivative estimator.
Empirically, we show that ReMAC achieves performance comparable to SAC while exhibiting higher entropy for $M>1$ than for $M=1$.

\paragraph{Contributions.}
Our main contribution is to identify intrinsic properties of the ReMax gradient, in both direction and magnitude, that can sustain policy stochasticity when $M>1$.
We illustrate this effect via a toy vector-field analysis and provide a theoretical explanation in the isotropic Gaussian policy setting.
We then instantiate the ReMax objective in an off-policy actor--critic algorithm, ReMAC, which can be implemented with minimal modifications to SAC.
Finally, we provide empirical evidence that ReMAC achieves higher policy entropy with $M>1$ than with $M=1$, without entropy regularization, and attains performance comparable to SAC in continuous-control tasks.
\section{Related Work} \label{sec:related_work}

We review the related work to this paper.

\textbf{Exploration in Deep RL.}
Extensive research has addressed the exploration-exploitation tradeoff in reinforcement learning (RL).
Existing exploration strategies can be broadly categorized into two paradigms: \textit{stochastic exploration} and \textit{deep exploration}.
Stochastic exploration maintains policy diversity to prevent premature convergence to suboptimal behaviors.
In contrast, deep exploration seeks temporally coherent behaviors capable of uncovering informative, unvisited regions of the environment.
A prevalent method for encouraging stochastic exploration is to introduce an entropy bonus into the policy optimization objective \citep{schulman2017proximal,haarnoja2018sac,christodoulou2019soft,lee2024simba,nauman2024bigger}.
This entropy regularization is a cornerstone of highly successful RL algorithms, such as Soft Actor-Critic (SAC), and remains a standard design choice in modern deep RL architectures \citep{nauman2024bigger,lee2024simba}.

Conversely, deep exploration is typically driven by intrinsic rewards, uncertainty estimation, or structured noise injection.
Standard intrinsic reward techniques include count-based exploration \citep{bellemare2016unifying,ostrovski2017count,lobel2023flipping,tang2017exploration}, prediction-error bonuses \citep{burda2018exploration,pathak2017curiosity}, and information gain or curiosity signals \citep{houthooft2016vime,sukhija2024maxinforl,schmidhuber2010formal}.
Alternatively, exploration can be promoted by modeling uncertainty in value estimates.
For instance, ensemble-based methods approximate posterior sampling using multiple Q-functions \citep{osband2016deep,osband2018randomized,osband2019deep,osband2023approximate}, while related approaches leverage Langevin Monte Carlo updates \citep{ishfaq2024more}.
In continuous control, Q-function ensembles are also used to reduce estimation bias \citep{chen2021randomized,hiraoka2021dropout}, and actor ensembles to search for high-value actions \citep{zhang2019ace}.
Finally, structured noise injection, whether in the action space or parameter space, offers another mechanism for temporally extended exploration \citep{lillicrap2020continuous,plappert2017parameter,fortunato2017noisy}.

\textbf{Retry-based objectives.}\space\space
Retry-based objectives offer a different route to exploration.
Rather than augmenting the reward function or injecting explicit noise into the policy, these objectives evaluate the maximum return achieved across multiple sampled trajectories.
Prior work demonstrates that retry objectives successfully induce exploratory behavior in discrete action spaces \citep{koyamada2023emergence,nishimori2026emergence,walder2025pass,chen2025pass,tang2025optimizing}.
\citet{koyamada2023emergence} first introduced ReMax and developed a policy-gradient estimator that relies on a resettable simulator.
Later, \citet{nishimori2026emergence} developed a Q-function-based formulation enabling analytical gradient computation in discrete action spaces.
In  large language model (LLM) post-training, several recent studies optimize pass@K, or its generalization, max@K, to enhance the diversity of generated outputs \citep{tang2025optimizing,chen2025pass,walder2025pass,bagirov2025bonworlds,takashiro2026maxk}.

The exploration property of ReMax has been studied by previous work \citep{tong2026finite,nishimori2026emergence} from the perspectives of both stochastic exploration and deep exploration.
Specifically, structured exploration emerges when ReMax is paired with posterior sampling; notably, \citet{tong2026finite} provided a finite-time regret analysis for ReMax in multi-armed bandits.
Conversely, stochastic exploration in ReMax has primarily been illustrated through a toy deterministic-reward bandit experiment, where it was observed that ReMax reshapes the gradient landscape and slows convergence, thereby encouraging exploration \citep{nishimori2026emergence}.
This was also demonstrated to be beneficial for deep RL training performance \citep{nishimori2026emergence}.
However, this prior analysis remains restricted to discrete action spaces and lacks a rigorous, detailed examination of the gradient landscape.

In this paper, we investigate the emergence of \textit{stochastic} exploration from ReMax in continuous action spaces with deterministic rewards, leaving the study of deep exploration for future work.
Unlike the pass@K literature, which assumes explicit retries for a given prompt and directly observes outcomes, our setting deals with episodic returns where retries must be emulated via function approximation.
Furthermore, while pass@K studies typically rely on sample-based REINFORCE-style estimators, we adopt a pathwise gradient estimator \citep{kingma2013auto,parmas2021unified,jankowiak2018pathwise,rezende2014stochastic}.
This continuous-space pathwise estimator is not directly applicable to standard LLM post-training, which relies on non-differentiable, rule-based reward verifiers.
Crucially, unlike traditional entropy regularization, ReMax alters the evaluation of rewards without introducing an auxiliary entropy bonus.
Consequently, it naturally avoids bonus-induced bias at the target optimum.

Recently, \citet{onoda2026multi} introduced multi-axis max@K, where the reward distribution in ReMax is modeled as a uniform distribution over criteria corresponding to different modes.
\citet{parmas2026ordergrad} extend maximum aggregation to general order-statistic objectives and provide a unified view of their policy-gradient estimators.
Related set-RL methods optimize objectives jointly over multiple sampled trajectories \citep{hamid2025polychromic,puri2026escaping,li2026setpo}, while other recent work considers retry-style objectives with vector-valued or uncertain rewards \citep{bahlousboldi2026vpo,gx2026using,onoda2026multi}.
Together, these developments suggest a broader view of retry objectives as tools for controlling the diversity and aggregation of multiple policy samples.

\section{Preliminaries}

\paragraph{Reinforcement learning (RL).}
We consider the reinforcement learning (RL) problem \citep{sutton1998intro}, formulated as a Markov decision process (MDP), $\mathcal{M} = (\mathcal{S}, \mathcal{A}, \mathcal{P}, p_0, \mathcal{R}, \gamma)$, where $\mathcal{S}$ is the state space, $\mathcal{A} \subseteq \mathbb{R}^d$ is the continuous action space with dimension $d$, $\mathcal{P}$ is the transition dynamics, $p_0$ is the initial state distribution, $\mathcal{R}: \mathcal{S} \times \mathcal{A} \to \mathbb{R}$ is the reward function, and $\gamma \in [0, 1)$ is the discount factor.
The goal is to train a policy $\pi(\cdot\mid s)$ that maximizes the expected discounted reward $\E[\pi]{R_0}$, where $R_0 = \sum_{h=0}^{\infty} \gamma^{h} r(s_h, a_h)$ is the discounted return and $\E[\pi]{\cdot}$ denotes the expectation over the trajectory $(s_0, a_0, r_0, s_1, a_1, r_1, \ldots)$ under policy $\pi$.
We define the value function of the policy $\pi$ as $V^\pi(s) = \E[\pi]{R_0 \mid s_0 = s}$ and the Q-function as $Q^\pi(s, a) = \E[\pi]{R_0 \mid s_0 = s, a_0 = a}$.

\section{ReMax Objective and its Gradient Estimation} \label{sec:remax}
This section introduces the ReMax objective and describes how we estimate its policy gradient.

\subsection{ReMax Objective} \label{sec:remax:remax_objective}
Optimizing returns over multiple trajectories can promote exploration without explicit exploration bonuses
\citep{walder2025pass,nishimori2026emergence,koyamada2023emergence}.
In episodic RL, directly optimizing such retry-based trajectory objectives often assumes the ability to reset the environment to arbitrary states
\citep{koyamada2023emergence}, which is typically infeasible in continuous-control tasks.
We therefore adopt the Q-function-based deterministic ReMax formulation
\citep{nishimori2026emergence}.
Given a Q-function $Q: \mathcal{S} \times \mathcal{A} \to \mathbb{R}$, the (deterministic) ReMax objective is defined as
\begin{equation}\label{eq:remax_objective}
    \mathcal{J}^M(\pi,Q)
    :=
    \E[s \sim p_0]{
    \mathcal{J}^M(\pi,Q, s)
    }, \quad \text{where} \quad
    \mathcal{J}^M(\pi,Q, s)
    :=
    \E[a_{1:M} \overset{\iid}{\sim} \pi(\cdot|s)]{
        \max_{m=1:M} Q(s,a_m)
    }
\end{equation}
where $M \in \mathbb{N}$ is the retry budget and $\pi$ is the policy.
When $M=1$ and $Q=Q^\pi$, this objective $\mathcal{J}^1(\pi,Q^\pi)$ reduces to the standard RL objective.

Importantly, the ReMax objective preserves the deterministic optimal policy.
Since an MDP admits an optimal deterministic policy under standard conditions \citep{sutton1998intro}, assume that there exists a deterministic optimal policy $\pi^*$ satisfying $Q^{\pi^*}=Q^*$ ($Q^*(s, a) := \sup_\pi Q^\pi(s,a)$) and $\pi^*(s) \in \arg\max_{a \in \mathcal{A}} Q^*(s,a)$ for every state $s \in \mathcal{S}$.
Then $\pi^* \in \arg\max_{\pi} \mathcal{J}^M(\pi,Q^\pi, s)$ for any finite $M \ge 1$.
Indeed, for any state $s$ and policy $\pi$,
\[
    \E[a_{1:M} \overset{\iid}{\sim} \pi(\cdot|s)]{
        \max_{m=1:M} Q^\pi(s,a_m)
    }
    \le
    \sup_a Q^*(s,a)
    =
    \E[a_{1:M} \overset{\iid}{\sim} \pi^*(\cdot|s)]{
        \max_{m=1:M} Q^*(s,a_m)
    }.
\]
Therefore, we have $\mathcal{J}^M(\pi, Q^\pi) \leq \mathcal{J}^M(\pi^*, Q^{\pi^*})$.
In contrast, entropy-regularized objectives shift the optimum toward higher-entropy policies and hence bias the optimum.

\subsection{Gradient Estimation of the ReMax Objective} \label{sec:remax:gradient_estimation}
We consider optimizing Eq.~\eqref{eq:remax_objective} using policy gradients.
\citet{nishimori2026emergence} showed that in discrete action spaces, both the objective and its gradient can be computed exactly, leveraging the sorted Q-values.
This approach does not extend to continuous action spaces because sorting Q-values over an uncountable action set is not feasible.
We therefore use sample-based estimators \citep{parmas2018pipps,parmas2018total}.
Let $\pi_\theta$ be a policy parameterized by $\theta \in \mathbb{R}^p$, where $p$ is the dimension of the policy parameters.
For a state $s$, we sample $a_{1:B} := (a_1, \ldots, a_B) \overset{\iid}{\sim} \pi_\theta(\cdot\mid s)$ with $B \ge M$ and compute $q_i = Q(s,a_i)$.
In practice, $Q$ is approximated by a neural network $Q_\phi$ with parameters $\phi$.
We consider two gradient estimators.

\textbf{Likelihood ratio (LR) gradient estimator.}
The likelihood-ratio estimator approximates the gradient of $\E[a\sim\pi_\theta]{f(a)}$ ($f:\mathcal{A}\to\mathbb{R}$) as $\E[a\sim\pi_\theta]{\nabla_\theta \log \pi_\theta(a\mid s) f(a)}$, also known as the REINFORCE estimator \citep{Williams1992}.
Because Eq.~\eqref{eq:remax_objective} involves a maximum over $M$ coupled action samples, the single-sample LR form above is not directly applicable as written.
Nevertheless, \citet{walder2025pass} derived an LR-style estimator for deterministic ReMax:
$\hat{g}=\frac{1}{B}\sum_{i=1}^B \nabla_\theta \log \pi_\theta(a_i\mid s)\,\widehat{\mathrm{Adv}}_i$,
where $\widehat{\mathrm{Adv}}_i$ is an advantage estimate.
We refer the reader to \citet{walder2025pass} for details.

\textbf{Reparameterization (RP) gradient estimator.}
Alternatively, we can estimate the gradient via the reparameterization trick \citep{kingma2013auto}.
Assume $\pi_\theta$ is Gaussian with mean $\mu$ and standard deviation $\sigma$ ($\theta=(\mu,\sigma)$).
Actions are sampled as $a_i(\mu,\sigma)=\mu+\sigma\xi_i$, where $\xi_i\sim\mathcal{N}(0,1)$.
Thus $\frac{\partial a_i}{\partial\mu}=1$ and $\frac{\partial a_i}{\partial\sigma}=\xi_i$.
For a $d$-dimensional action, $\mu,\sigma\in\R^d$ and $a_i=\mu+\sigma\odot\xi_i$ with $\xi_i\sim\cN(0,I_d)$; the $\xi_{1:B}:=( \xi_1, \ldots, \xi_B)$ of Eq.~\eqref{eq:actor_loss} are then $B$ noise vectors in $\R^d$.
When $Q$ is differentiable, we can compute $\nabla_{a_i}Q(s,a_i)$ and backpropagate through $q_i$.
\citet{walder2025pass} proposed an unbiased estimator of Eq.~\eqref{eq:remax_objective} (for the given $Q$) from $q_{1:B} := (q_1, \ldots, q_B)$ with order statistics $q_{(1)}\le\cdots\le q_{(B)}$ (ties broken randomly):
\begin{equation} \label{eq:remax_objective_rp}
    \rho^M(q_{1:B}) := \frac{1}{\binom{B}{M}} \sum_{i=1}^B w_i q_{(i)},
\end{equation}
where $w_i = \binom{i-1}{M-1}$ counts the number of size-$M$ subsets for which the $i$-th sorted action is the best; the factor $\binom{B}{M}^{-1}$ converts this count into a probability.
Because $\rho^M$ is linear in the sorted values $q_{(1)},\ldots,q_{(B)}$, $\nabla_\theta \rho^M(q_{1:B})$ can be obtained by automatic differentiation.
\citet{walder2025pass} did not explore this RP estimator because their focus was LLM post-training, where rewards are typically non-differentiable.

\textbf{Our choice.}
The two estimators have different trade-offs.
LR gradients can remain unbiased with raw returns, which suits on-policy learning, but they often suffer from high variance \citep{greensmith2004variance}.
RP gradients require a differentiable Q-function but typically (but not always) yield lower-variance estimates \citep{parmas2018pipps}.
In continuous control, off-policy actor–critic methods \citep{degris2012off,haarnoja2018sac,fujimoto2018addressing} are widely used for their sample efficiency.
These methods rely on replay buffers that store transitions from past behavior policies and update both the Q-function and policy off-policy.
As a result, obtaining raw-return targets for arbitrary state–action pairs is difficult, which weakens the main advantage of LR estimators.
Thus, in this setting, LR offers limited benefits while retaining high variance, whereas RP estimators are standard and effective.
We therefore adopt the RP estimator.

\section{Gradient Analysis} \label{sec:grad_analysis}

We analyze how the ReMax objective shapes policy gradients.
Even with deterministic rewards, ReMax can promote higher policy stochasticity by reshaping the gradient landscape.
In a stylized fixed-state setting, we show that, for $M>1$ and a sufficiently small policy scale (the standard deviation $\sigma$ of the Gaussian policy), the ReMax gradient increases that scale whenever the policy mean is non-stationary.
We also derive gradient bounds controlled by the distance of the best sampled action to the optimum: they vanish at the deterministic optimum, and the mean-gradient bound gives an upper envelope that decreases with $M$.
We first illustrate both effects in a toy vector field (Sec.~\ref{sec:grad_analysis:vector_field}), then establish them (Sec.~\ref{sec:grad_analysis:theoretical_explanation}), and finally show that Adam's normalization can mitigate the damping effect (Sec.~\ref{sec:grad_analysis:adam_effect}), offering practical guidance for continuous control.

\paragraph{Notation.}
The action space is $\cA\subseteq\R^{d}$, and $d$ denotes the action dimension.
The policy is an isotropic Gaussian $\cN(\mu,\sigma^2 I_d)$ with mean $\mu\in\R^{d}$ and policy scale $\sigma\in\R_{>0}$, a single scalar shared across action dimensions.
Accordingly, $\nabla_\mu$ is a gradient with respect to a vector, measured in the Euclidean norm $\|\cdot\|$, and $\partial_\sigma$ is a derivative with respect to a scalar, so $\partial_\sigma J^M$ is a real number.
Action samples are written $A_m$ and reparameterization noise $\xi_m\sim\cN(0,I_d)$.
We also write $Z_1,\ldots,Z_M\overset{\iid}{\sim}\cN(0,1)$ and $\beta_M:=\E{\max_{m\le M}Z_m}$, so that $\beta_1=0$ and $\beta_M>0$ for $M\ge2$.

The analyzed objective is the fixed-state, fixed-critic analogue of the actor objective that ReMAC differentiates.
For a state $s$ and a fixed critic $Q$, write $c(a):=-Q(s,a)$; the object we study is then the minimization of a deterministic cost $c:\R^d\to\R$ of the action alone, so the state dependence of $Q(s,a)$ is absorbed into the choice of $c$.
ReMAC instead differentiates an expectation over states drawn from the replay buffer, with a critic that is simultaneously being fit and a state-dependent diagonal scale.
The gradient with respect to the actual network parameters $\theta$ follows by the chain rule through the Jacobian of $(\mu_\theta(s),\sigma_\theta(s))$, so its norm is bounded by the operator norm of that Jacobian times the norms of the $\mu$- and $\sigma$-gradients; the mechanisms studied here therefore transfer to $\theta$ as long as this Jacobian stays bounded, and for multiple states the objective averages the fixed-state landscape over the state distribution.

To illustrate this effect, we consider a one-dimensional ($d=1$) Gaussian policy with mean $\mu \in \mathbb{R}$ and standard deviation $\sigma > 0$, i.e., $a \sim \mathcal{N}(\mu,\sigma^2)$.
We use the quadratic reward $r(a) = -a^2$.
\subsection{Vector field analysis} \label{sec:grad_analysis:vector_field}
In Fig.~\ref{fig:vector_field}, we visualize the vector field of the RP gradient estimator (Sec.~\ref{sec:remax:gradient_estimation}) with batch size $B=16$ over $\mu\in[-2,2]$ and $\sigma\in[0.1,2]$, for $M\in\{1,4,8\}$.
For reference, we also plot the entropy-regularized objective for $M=1$, $J_{\mathrm{ent}}(\mu,\sigma)=\mathbb{E}_{a\sim\mathcal{N}(\mu,\sigma^2)}[r(a)]+\eta\,\mathcal{H}(\mathcal{N}(\mu,\sigma^2))$ with $\eta=0.1$, where $\mathcal{H}(\mathcal{N}(\mu,\sigma^2))=\frac12\log(2\pi e\sigma^2)$.
For $M=1$, the gradient with respect to $\sigma$ is always negative.
For $M>1$, ReMax exhibits two exploration-promoting properties:

\begin{itemize}
\item \textbf{Entropy Increase (direction):}
When the mean is far from the optimum and $\sigma$ is small, the gradient increases $\sigma$, inducing stochastic exploration.
\item \textbf{Gradient Damping (magnitude):}
As $M$ increases, the gradient norm near the optimal mean (small $\mu$ and large $\sigma$) decreases, slowing convergence and helping sustain larger entropy.
\end{itemize}

More details are provided in App.~\ref{sec:app:gradient:vector_field}.
The maximizer of the ReMax objective remains the deterministic optimum $(\mu,\sigma)=(0,0)$, whereas entropy regularization shifts the maximizer to a point with $\sigma>0$.
Together, these effects can keep entropy high in gradient ascent: it first increases due to the directional effect and then decreases gradually due to damping.
Sec.~\ref{sec:grad_analysis:adam_effect} shows how Adam's normalization changes the second phase.

\subsection{Theoretical Explanation} \label{sec:grad_analysis:theoretical_explanation}

We now provide intuition for how ReMax with $M>1$ shapes exploration dynamics in high-dimensional action spaces.
The vector-field analysis in Sec.~\ref{sec:grad_analysis:vector_field} is one-dimensional, but the underlying mechanism is not specific to one dimension.
In this section, we provide a theoretical explanation for the entropy-increase and gradient-damping behaviors observed in the vector field.

\paragraph{Setup and assumptions.}

Let $\mu\in\mathbb R^d$ and $\sigma>0$ denote the mean and scalar scale of an isotropic Gaussian policy.
For $m=1,\ldots,M$, let
\[
    A_m=\mu+\sigma\xi_m,
    \qquad
    \xi_m\stackrel{\iid}{\sim}\mathcal N(0,I_d).
\]
We consider deterministic rewards of the form $r(a)=-c(a)$ with $c:\mathbb R^d\to\mathbb R$.
For $M\ge1$, define
$J^M(\mu,\sigma) = \mathbb E\left[ \max_{m=1,\ldots,M} -c(A_m) \right]$.
Equivalently,
$J^M(\mu,\sigma) = - \mathbb E\left[ \min_{m=1,\ldots,M} c(A_m) \right]$.
Let
$m^\star\in\arg\min_{m=1,\ldots,M}c(A_m)$ denote the selected action.
These assumptions are introduced for analytical tractability rather than to model deep RL environments globally, allowing us to study how ReMax shapes the local gradient landscape.

\begin{assumption}[Smoothness]\label{ass:smoothness}
$c$ is differentiable and $L$-smooth, i.e.,
$\|\nabla c(a)-\nabla c(b)\| \le L\|a-b\|$ for all $a,b\in\mathbb R^d$.
\end{assumption}

\begin{assumption}[Non-degenerate selection]\label{ass:non_degenerate_selection}
Every level set $\{a:c(a)=t\}$, $t\in\mathbb R$, has Lebesgue measure zero.
\end{assumption}

This assumption ensures that the selected sample is almost surely unique.
For every $\mu$ and every $\sigma>0$, this assumption is equivalent to $c(A_m)$ having no atoms, since $A_m$ has a positive density on $\mathbb R^d$.
As shown in App.~\ref{sec:app:proofs_of_propositions_in_sec_grad_analysis}, this condition is automatically satisfied under strong convexity (Assumption~\ref{ass:strong_convexity}).

\begin{assumption}[Centered optimum]\label{ass:centered_optimum}
The cost is normalized so that $c(0)=0$ and $\nabla c(0)=0$.
\end{assumption}

\begin{assumption}[Strong convexity]\label{ass:strong_convexity}
$c$ is $\lambda$-strongly convex, i.e.,
\[
    c(b)
    \ge
    c(a)+\nabla c(a)^\top(b-a)+\frac{\lambda}{2}\|b-a\|^2
\]
for all $a,b\in\mathbb R^d$.
\end{assumption}

\begin{assumption}[Second-order regularity]\label{ass:second_order_regularity}
$c$ is twice continuously differentiable.
\end{assumption}

The entropy-increase result does not require convexity: it relies only on global \(L\)-smoothness and non-degenerate sample selection.
The entropy-decrease result for \(M=1\) and the gradient-damping result additionally use strong convexity, together with second-order regularity and the centered-optimum normalization, respectively.
These are standard assumptions in first-order optimization theory~\citep{nesterov2013introductory,bubeck2015convex} rather than the weakest possible conditions, and although they do not fully reflect practical deep-RL settings with neural networks---in particular, we analyze the actor's objective with $Q$ held fixed, rather than the coupled dynamics in which the critic is simultaneously re-fit---they provide a clean first setting in which the two mechanisms observed in the vector field can be isolated and proved rigorously.
\paragraph{Entropy increase effect for $M>1$.}

We first show that ReMax can increase the policy scale when the current mean is not stationary and the scale is small.
This result does not require global convexity.

\begin{proposition}[Entropy Increase Effect]\label{prop:entropy_increase}
Assume Assumptions~\ref{ass:smoothness} and~\ref{ass:non_degenerate_selection}.
Let $M\ge2$ and suppose that $\nabla c(\mu)\neq0$.
Then, there exists $\sigma_0>0$ such that
\[
    \partial_\sigma J^M(\mu,\sigma)>0
    \qquad
    \text{for all }\sigma\in(0,\sigma_0).
\]
\end{proposition}

The proof is given in App.~\ref{sec:app:proof:entropy_increase}.
The intuition is that ReMax can exploit upward fluctuations in reward across multiple samples.
Since $r(a)=-c(a)$, this is equivalent to selecting a sample with a downward fluctuation in cost.
Let $g=\nabla c(\mu)$.
For small $\sigma$, the cost of each sampled action is locally approximated as $c(\mu+\sigma\xi_m)=c(\mu)+\sigma g^\top\xi_m+O(\sigma^2\|\xi_m\|^2)$.
Thus, to first order, the selected action is the one with the smallest projection $g^\top\xi_m$, or equivalently the largest reward fluctuation $-g^\top\xi_m$.
For a single sample, this fluctuation has zero mean.
With $M\ge2$, however, taking the maximum over samples creates a positive expected upward fluctuation, since $\mathbb E[\max_m(-g^\top\xi_m)]=-\mathbb E[\min_m g^\top\xi_m]=\|\nabla c(\mu)\|\,\beta_M>0$, where $\beta_M=\mathbb E[\max_{m\le M}Z_m]$ and $Z_1,\ldots,Z_M\overset{\iid}{\sim}\cN(0,1)$.
Therefore, when $\sigma$ is small, increasing the policy scale makes these favorable fluctuations more available, yielding a positive scale gradient.
This is the high-dimensional analogue of the entropy-increase direction observed in Sec.~\ref{sec:grad_analysis:vector_field}.

Previous work in evolution strategies examined an update rule based on selecting the max from a set of samples, where $\beta_M$ appears as the progress coefficient, and found that, on the sphere model, the rate of progress improves with an increased $\sigma$ when $\sigma$ is small \citep{beyer2001theory,beyer2002es}; their focus is different, and they do not discuss how optimizing the ReMax objective changes the learning dynamics or can increase $\sigma$.
\paragraph{Entropy decrease for $M=1$.}

The previous effect contrasts with the standard single-sample objective.
When $M=1$, increasing the policy scale only injects noise into the cost, and the scale gradient is always negative under strong convexity.

\begin{proposition}[Entropy Decrease for $M=1$]\label{prop:variance_decrease}
Assume Assumptions~\ref{ass:smoothness}, \ref{ass:strong_convexity}, and~\ref{ass:second_order_regularity}.
For $M=1$, $\mu\in\mathbb R^d$, and $\sigma>0$,
\[
    \partial_\sigma J^1(\mu,\sigma)
    \le
    -\left(\lambda d\right)\,\sigma
    <
    0.
\]
\end{proposition}

This is almost immediate from Stein's lemma \citep{stein1981estimation}; we include a proof in App.~\ref{sec:app:proof:variance_decrease} for completeness and as the baseline against which Proposition~\ref{prop:entropy_increase} is read.
This result explains why the $M=1$ vector field consistently shrinks the policy scale.
Without best-of-$M$ selection, the objective directly penalizes variance around the current mean.

\paragraph{Gradient damping effect for $M>1$.}

We next bound the gradient magnitudes by the expected distance of the best sampled action to the optimum.

\begin{proposition}[Gradient Damping]\label{prop:gradient_damping}
Assume Assumptions~\ref{ass:smoothness}, \ref{ass:centered_optimum}, and~\ref{ass:strong_convexity}.
Let $M\ge1$ and $\sigma>0$.
Then, the gradients exist and satisfy
\[
    \|\nabla_\mu J^M(\mu,\sigma)\|
    \le
    L\sqrt{\frac{L}{\lambda}}\,
    \mathbb E\left[
        \min_{m=1,\ldots,M}\|A_m\|
    \right],
\]
and
\[
    |\partial_\sigma J^M(\mu,\sigma)|
    \le
    L\sqrt{\frac{L}{\lambda}}\,
    \mathbb E\left[
        \min_{m=1,\ldots,M}\|A_m\|
        \|\xi_{m^\star}\|
    \right].
\]
\end{proposition}

The proof is given in App.~\ref{sec:app:proof:gradient_damping}.
Since $A_m=\mu+\sigma\xi_m$, both bounds vanish as $(\mu,\sigma)\to(0,0)$.

These bounds capture the ``forgiveness'' of the minimum operator.
In standard RL ($M=1$), a large policy scale near the optimal mean produces poor actions and strong negative scale gradients that rapidly shrink $\sigma$.
In contrast, ReMax evaluates only the best action.
When $\mu$ is near $0$, increasing $M$ raises the probability that at least one action lies close to the optimum, reducing the selected deviation and damping gradients.
In Proposition~\ref{prop:gradient_damping}, this reduction appears as $\min_m\|A_m\|$, which is non-increasing in $M$.
Corollary~\ref{cor:damping_explicit} in App.~\ref{sec:app:proof:damping_explicit} makes the dependence on $M$ explicit.
When the policy mean is far from stationary and $\sigma$ is small, the projection effect in Proposition~\ref{prop:entropy_increase} dominates and increases the policy scale.

\subsection{Effect of Adam's normalization} \label{sec:grad_analysis:adam_effect}
\begin{figure}[t]
    \centering
    \includegraphics[width=1.0\textwidth]{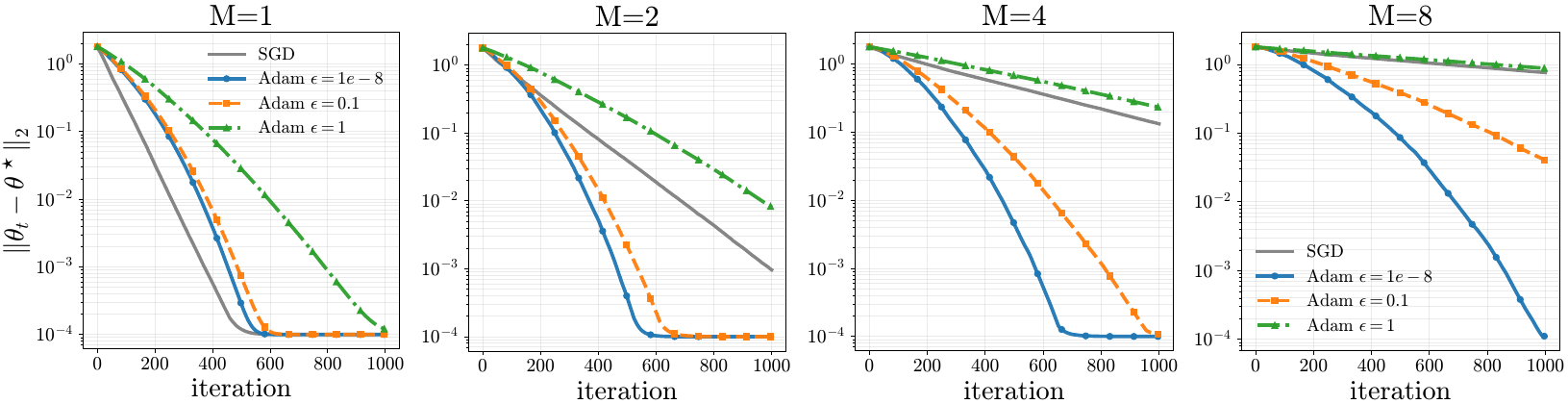}
    \caption{
        The distance to the optimal parameter $(\mu, \sigma) = (0, 0)$ for Adam with different $\epsilon$ values and for SGD.
        As the vector field indicates, the convergence speed slows as $M$ increases in every case.
        However, with smaller $\epsilon$ the convergence near the optimum remains fast, whereas, for $M\ge4$, larger $\epsilon$ aligns the trajectory with the SGD path that follows the vector field.
    }
    \label{fig:adam_effect}
\end{figure}

Although the vector-field analysis suggests that ReMax with $M>1$ promotes stochastic exploration, the actual optimization trajectory need not follow the raw vector field.
With Adam \citep{kingma2014adam}, adaptive normalization can counteract gradient damping, leading to faster convergence and faster entropy collapse.
Adam updates parameters as
$\theta_{t+1} = \theta_t - \alpha \frac{\hat m_t}{\sqrt{\hat v_t}+\epsilon}$, where $\hat m_t$ and $\hat v_t$ are the bias-corrected first and second moments.
If $\epsilon$ is small, $\frac{\hat m_t}{\sqrt{\hat v_t}+\epsilon} \approx \frac{\hat m_t}{\sqrt{\hat v_t}}$, effectively normalizing gradients and reducing damping.
If instead $\sqrt{\hat v_t}\ll\epsilon$, the update is approximately $\hat m_t/\epsilon$, so the large-$\epsilon$ regime approaches \emph{momentum} SGD with effective learning rate $\alpha/\epsilon$---not plain SGD---and the gradient magnitude is transmitted to the parameters.
In the toy experiment below, $\epsilon=1$ empirically tracks the plain-SGD baseline for $M\ge4$, but the value $1$ has no special theoretical status.
Varying $\epsilon$ rescales the denominator but never removes it; an ablation that deletes it entirely is reported in App.~\ref{sec:app:gradient:adam_effect}.

To illustrate this effect, we optimize $a \sim \mathcal{N}(\mu, \sigma^2)$ with $r(a) = -a^2$ using Adam with $\epsilon \in \{10^{-8}, 10^{-1}, 1.0\}$.
We set $\alpha=0.01$, $B=64$, $T=1000$, and $M\in\{1,2,4,8\}$, starting from $(\mu,\sigma)=(-1.5,1.0)$, with $\sigma$ clipped from below at $10^{-4}$.
Fig.~\ref{fig:adam_effect} shows the distance to $(0,0)$ for Adam and SGD.
As expected, convergence slows as $M$ increases.
However, smaller $\epsilon$ mitigates damping: with $M=4,8$, $\epsilon=10^{-8}$ converges much faster than SGD, whereas $\epsilon=1.0$ aligns with SGD.
These results suggest a practical guideline.
Adam benefits from the directional entropy-increase effect, but with the default $\epsilon=10^{-8}$, damping is weakened.
Increasing $\epsilon$ restores damping but reduces the effective step size.
Jointly tuning $\epsilon$ and the learning rate may therefore better control optimization dynamics during training.
More details are provided in App.~\ref{sec:app:gradient:adam_effect}.

\section{ReMax Actor-Critic Algorithm} \label{sec:algorithm}
\begin{algorithm}[t]
    \caption{ReMAC}
    \label{alg:remax_actor_critic}
    \begin{algorithmic}[1]
    \REQUIRE Retry parameter $M$, action batch size $B$, discount $\gamma$, learning rate $\alpha$, target critic update rate $\tau$.
    \STATE Initialize policy $\pi_\theta$, critics $Q_{\phi_1},Q_{\phi_2}$, target critics $Q_{\bar\phi_1},Q_{\bar\phi_2}$, and replay buffer $\mathcal D$

    \FOR{each iteration $t$}

    \FOR{each environment step}
    \STATE $a \sim \pi_\theta(\cdot|s),\ (r,s')\sim\mathcal P(s,a)$ \RightComment{interact with environment}
    \STATE $\mathcal D \leftarrow \mathcal D \cup \{(s,a,r,s')\}$ \RightComment{store transition}
    \ENDFOR

    \FOR{each gradient step}
    \STATE $(s,a,r,s') \sim \mathcal D,\ a' \sim \pi_\theta(\cdot|s')$ \RightComment{sample minibatch and next action}
    \STATE $y \leftarrow r+\gamma \min(Q_{\bar\phi_1}(s',a'),Q_{\bar\phi_2}(s',a'))$ \RightComment{target value}
    \STATE $\phi_j \leftarrow \phi_j-\alpha\nabla_{\phi_j}(Q_{\phi_j}(s,a)-y)^2,\ j\in\{1,2\}$ \RightComment{update critics}
    \STATE $a_{1:B}\sim\pi_\theta(\cdot|s),\ q_{i}\leftarrow \min_{j=1,2} Q_{\phi_j}(s,a_{i}) \forall i \in \{1,2,\ldots,B\}$ \RightComment{sample actions and evaluate Q}
    \STATE $\theta \leftarrow \theta+\alpha\nabla_\theta \rho^M(q_{1:B})$ \RightComment{update policy}
    \STATE $\bar\phi_j \leftarrow \tau\phi_j+(1-\tau)\bar\phi_j,\ j\in\{1,2\}$ \RightComment{update target critics}
    \ENDFOR

    \ENDFOR
    \end{algorithmic}
    \end{algorithm}
Having identified the policy-stochasticity effects of ReMax, we instantiate it in a simple off-policy actor--critic algorithm with a critic $Q_\phi$ and a stochastic policy $\pi_\theta$.
The critic is trained by temporal-difference regression on transitions sampled from a replay buffer $\mathcal{D}$, while the actor is optimized using ReMax computed from batches of action samples.

\textbf{Critic update.}
Given a transition $(s,a,r,s')\sim\mathcal{D}$, we define the target as $y = r + \gamma Q_{\bar\phi}(s', a')$, where $a'\sim \pi_\theta(\cdot\mid s')$ and $Q_{\bar\phi}$ is a target critic.
The critic is updated by minimizing the squared Bellman error:
\begin{equation}
    L_Q(\phi):=\mathbb{E}_{(s, a, r, s') \sim \mathcal{D}}\big[(Q_\phi(s,a)-y)^2\big].
\end{equation}

\textbf{Actor update.}
For policy optimization, we use a $\tanh$-squashed Gaussian policy $\pi_\theta(\cdot\mid s)$, as in SAC, whose mean $\mu_\theta(s)\in\R^d$ and diagonal scale $\sigma_\theta(s)\in\R^d_{>0}$ are neural-network outputs; unlike the scalar $\sigma$ of Sec.~\ref{sec:grad_analysis}, the scale is a state-dependent vector.
The $\tanh$ is a common way to keep actions within a bounded action space, and we use it in all our experiments.
We sample $B$ independent noise vectors $\xi_i\overset{\iid}{\sim}\cN(0,I_d)$ and construct the pre-squash actions as
\[
    \tilde a_i=\mu_\theta(s)+\sigma_\theta(s)\odot\xi_i,
    \qquad
    a_\theta(\xi_i)=\tanh(\tilde a_i),
    \qquad
    i=1,\ldots,B,
\]
where $\odot$ is the elementwise product.
We then evaluate the Q-values as $Q_\phi(s,a_\theta(\xi_{1:B})) := (Q_\phi(s, a_\theta(\xi_1)), \ldots, Q_\phi(s, a_\theta(\xi_B)))$.
The actor loss is
\begin{equation} \label{eq:actor_loss}
    L_{\text{actor}}^M(\theta)
    := -\mathbb{E}_{s\sim\mathcal{D},\ \xi_{1:B}\overset{\iid}{\sim}\mathcal{N}(0,I_d)}\big[\rho^M(Q_\phi(s,a_\theta(\xi_{1:B})))\big].
\end{equation}
We optimize the policy parameters by automatic differentiation of $L_{\text{actor}}^M(\theta)$.
We call this algorithm \textbf{ReM}ax \textbf{A}ctor-\textbf{C}ritic (\textbf{ReMAC}); the full procedure is summarized in Alg.~\ref{alg:remax_actor_critic}.

\textbf{Remarks on the implementation.}
To reduce overestimation, we use clipped double Q-learning \citep{haarnoja2018sac,fujimoto2018addressing} with two critics and take the minimum in both the target and actor loss.
In practice, we obtain ReMAC by removing the ``soft'' components of Soft Actor-Critic (SAC) \citep{haarnoja2018sac,haarnoja2018soft2}---namely, entropy bonuses in the critic targets and temperature tuning---and replacing the actor loss with our ReMax loss.
This is a minimal off-policy actor--critic instantiation of ReMax, and many extensions are possible.
One promising direction is to maintain a posterior over the Q-function \citep{osband2016deep,osband2018randomized,osband2019deep}, making return uncertainty explicit as in ReMax \citep{koyamada2023emergence,nishimori2026emergence}.
This may enable \textit{deeper}, more structured exploration, which is crucial in sparse-reward settings.

\section{Experiments} \label{sec:experiments}

The goal of our empirical evaluation is \textit{not} to establish ReMAC as a new state-of-the-art exploration method for continuous control; we intend it as a proof of concept, and ask two questions.
\begin{enumerate}
    \item Does ReMAC achieve performance comparable to SAC in continuous-control tasks?
    \item Does ReMAC exhibit higher policy entropy for $M>1$ than $M=1$ in practical deep-RL training?
\end{enumerate}

\textbf{Setup.}
We considered seven continuous-control tasks from Brax \citep{freeman2021brax}: Ant, HalfCheetah, Hopper, Reacher, Swimmer, Walker2d, and HumanoidStandup.
Agents were trained for 3 million environment steps.
We report the average return over 128 evaluation episodes, together with the standard error across 10 random seeds. In the evaluation, we sampled the actions from the stochastic policy, following the default of the Brax code\footnote{This is the default evaluation setting of Brax's training code (\url{https://github.com/google/brax}) and of \texttt{rejax} (\url{https://github.com/keraJLi/rejax}), on which our implementation builds.}.

\textbf{Baselines.}
We used SAC \citep{haarnoja2018sac} and PPO \citep{schulman2017proximal} as baselines to verify that ReMAC maintains competitive performance.
SAC is the main baseline, on which ReMAC is built, and we include PPO as an on-policy reference.
Our implementation built on \texttt{rejax} \citep{liesen2024rejax}, which provides vectorized training pipelines for SAC and PPO with tuned configurations for each Brax task; we adopted these configurations for the baselines.
SAC uses automatic temperature tuning \citep{haarnoja2018soft2}: the temperature is a learned parameter, initialized at $1$ and adjusted to hold the policy entropy at the target $-\dim(\cA)$.
We evaluated ReMAC with retry budgets \(M\in\{1,2,4,8\}\).
For \(M=1,2,4\), we set the action-sample batch size to \(B=8\), and for \(M=8\), we set \(B=16\).
For the main comparison, we used Adam's default \(\epsilon=10^{-8}\).
For more details on the training setup, see App.~\ref{sec:app:experiments:setup}.

As discussed in Sec.~\ref{sec:grad_analysis:adam_effect}, a small value of Adam's \(\epsilon\) can mitigate the gradient-damping effect by normalizing gradient magnitudes.
Thus, the main results with the default \(\epsilon=10^{-8}\) primarily examined whether the directional entropy-increase effect of ReMax appears in practical actor--critic training.
To study the role of Adam's \(\epsilon\) in the damping effect, we additionally varied the actor's \(\epsilon\in\{10^{-8},10^{-2},10^{-1},1.0\}\) for all \(M\) (the critic's stays at \(10^{-8}\)); the results are provided in App.~\ref{sec:app:experiments:additional_results:entropy}.
We also tested the sensitivity of the results to the action-sample batch size \(B\in\{8,16\}\), with results reported in App.~\ref{sec:app:experiments:bs}.

\begin{figure}[t]
    \centering
    \includegraphics[width=1.0\textwidth]{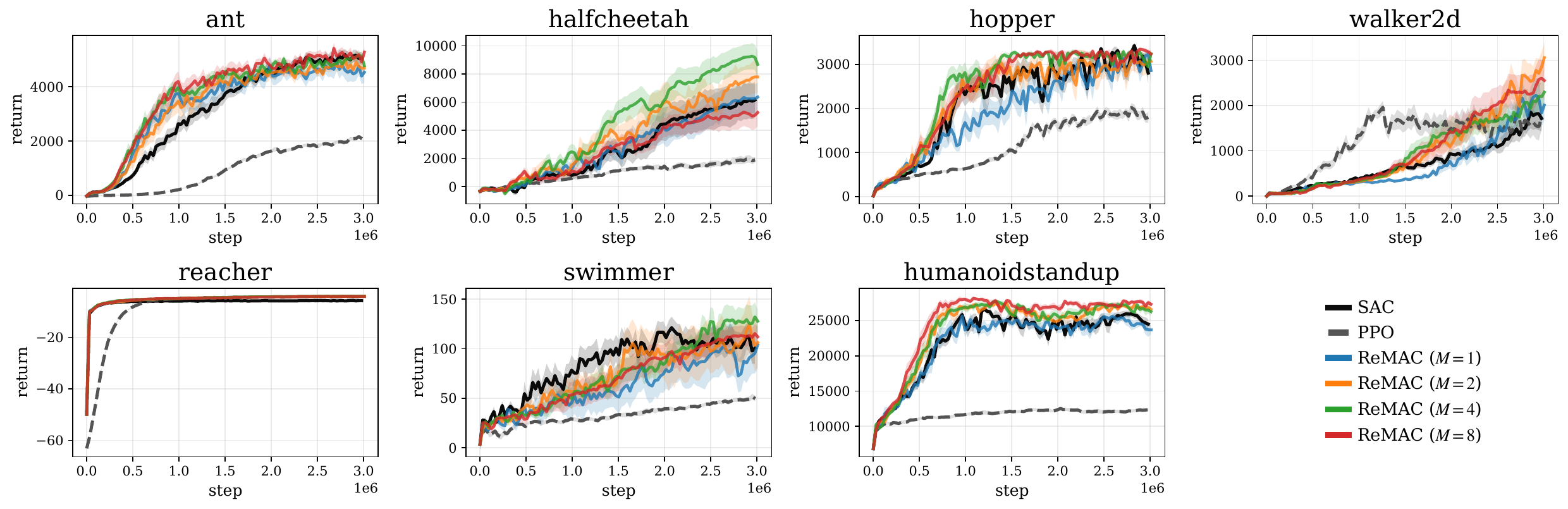}
    \caption{
        Average return of SAC, PPO, and ReMAC with different $M$ under Adam's default $\epsilon=10^{-8}$.
        The shaded area shows the standard error across 10 random seeds.
        ReMAC with $M>1$ achieves performance comparable to SAC.
    }
    \label{fig:main_result_return}
\end{figure}

\begin{figure}[t]
    \centering
    \includegraphics[width=1.0\textwidth]{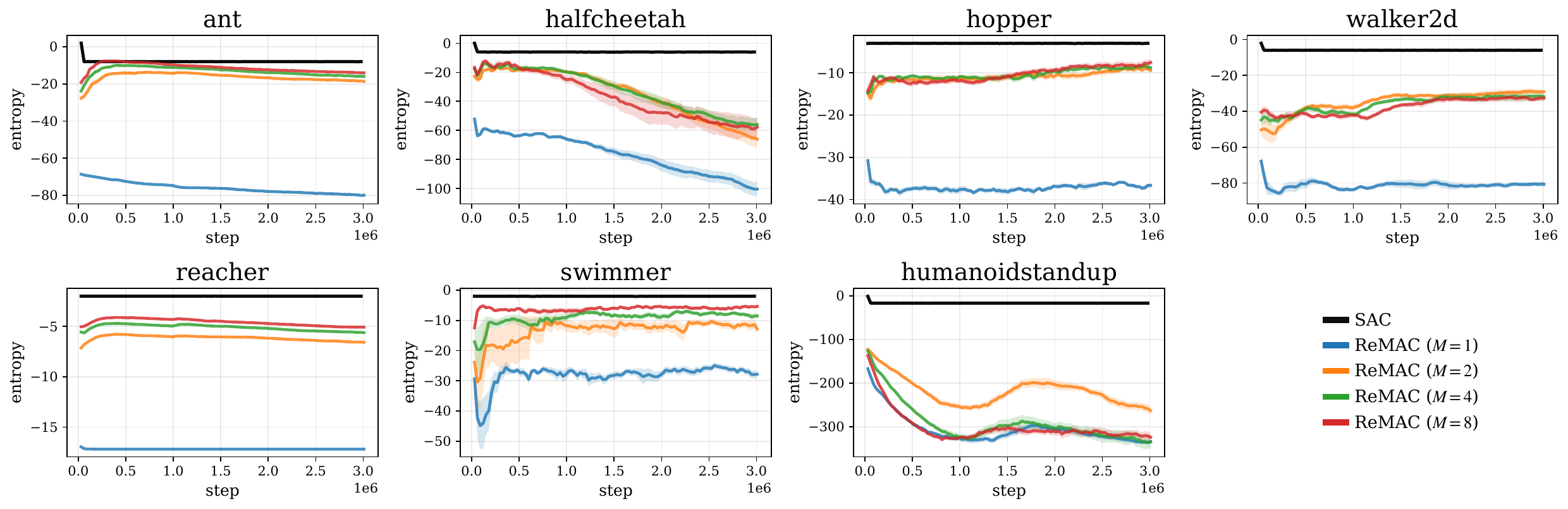}
    \caption{
        Average policy entropy of SAC and ReMAC with different $M$ under Adam's default $\epsilon=10^{-8}$.
        The shaded area shows the standard error across 10 random seeds.
    }
    \label{fig:main_entropy}
\end{figure}

\textbf{Main results.}
Fig.~\ref{fig:main_result_return} reports the average return.
Overall, ReMAC with \(M>1\) achieved average returns comparable to those of SAC and generally performed better than PPO.
In most tasks, larger retry budgets (\(M=2,4,8\)) matched or exceeded the average return of \(M=1\).

Fig.~\ref{fig:main_entropy} reports the average policy entropy for different \(M\).
ReMAC with \(M>1\) yielded higher entropy than \(M=1\) in six of the seven tasks, by a margin clearly larger than the spread across seeds; HumanoidStandup for $M>2$, whose actions saturate most strongly (App.~\ref{sec:app:experiments:mechanism}), is the one exception.
This behavior is consistent with the qualitative trend observed in the toy analysis of Sec.~\ref{sec:grad_analysis:vector_field}.

Early in training, entropy increases with \(M\) in those six tasks.
The analysis of Sec.~\ref{sec:grad_analysis} holds a state and a critic fixed, whereas runs with different \(M\) are trained separately and follow different dynamics, so the ordering need not persist once they separate.
The reported entropy is moreover that of the \(\tanh\)-squashed policy $\cH\bigl(\pi_\theta(\cdot\mid s)\bigr)
= \sum_{i=1}^{d}\tfrac12\log\bigl(2\pi e\,\sigma_\theta(s)_i^2\bigr)
+ \E[\tilde a]{\textstyle\sum_{i=1}^{d}\log\bigl(1-\tanh^2\tilde a_i\bigr)}$, whose second term depends on the policy mean as well as on the scale, so the analysis does not apply to it directly (App.~\ref{sec:app:experiments:entropy_def}); we report it because it is what determines the diversity of the actions the agent actually takes.
In App.~\ref{sec:app:experiments:entropy_def} (Fig.~\ref{fig:entropy_terms}), we report the two terms separately: the first term, which depends only on the scale, is far larger for every $M>1$ than for $M=1$, the direction of Proposition~\ref{prop:entropy_increase}.
The second term decreases during training in some tasks as the actions saturate toward the bounds of the action space; this is most pronounced in HumanoidStandup, where for $M\ge4$ it offsets the larger first term, so that the entropy stays close to that of $M=1$.
The effect of Adam's \(\epsilon\) is reported in App.~\ref{sec:app:experiments:additional_results:entropy}.
Finally, although ReMAC with \(M>1\) increased entropy relative to \(M=1\), its entropy remained lower than SAC.
SAC's entropy is a controlled setpoint: its automatic temperature tuning holds it at the target \(-\dim(\cA)\) \citep{haarnoja2018soft2}, which ReMAC does not do.
From a practical perspective, these environment-step comparisons should be interpreted together with the additional computational cost of ReMAC, which arises from the extra Q-network evaluations for \(B\) action samples per state; see App.~\ref{sec:app:experiments:cost} for details.

\textbf{Choosing the retry budget.}
No single $M$ is best across our environments, and performance is not monotone in $M$; nonetheless, $M>1$ generally yields higher entropy than $M=1$ and tends to match or improve on its return.
Across those six tasks, $M=8$ yields higher final policy entropy than $M=1$ at every tested $\epsilon$ (App.~\ref{sec:app:experiments:additional_results:entropy}), although the ordering across intermediate values of $M$ is not monotone.
We therefore treat $M$ as a hyperparameter that controls how strongly the objective favors stochasticity, much like the entropy coefficient of SAC---though, unlike an entropy bonus, it does not change the deterministic optimum of the idealized ReMax objective (Sec.~\ref{sec:remax:remax_objective}).
Since the cost of ReMAC grows with the action-sample batch size $B\ge M$, we recommend $M\in\{2,4\}$ as practical default choices.

\section{Conclusion}
In this study, we extended the ReMax objective to continuous action spaces.
We showed that ReMax reshapes the gradient landscape, guiding the policy toward higher entropy during optimization.
We first demonstrated this effect via a toy vector-field analysis and then provided a theoretical explanation.
On the practical side, we instantiated ReMax in an off-policy actor--critic algorithm, ReMax Actor-Critic (ReMAC), which can be implemented with minimal modifications to SAC.
Empirically, we showed that ReMAC can promote higher policy entropy with $M>1$ than with $M=1$, without entropy regularization, and achieves performance comparable to SAC.

\paragraph{Limitations and future work.}
Theoretically, the entropy-increase result assumes smoothness and non-degenerate sample selection, and the damping bounds additionally assume strong convexity of the fixed-state cost $c(a)=-Q(s,a)$, which is not expected to hold for a general learned neural-network critic; our analysis is intended to build intuition, and extending it beyond these conditions is left to future work.
Practically, our implementation is minimal and focuses on stochastic exploration (Sec.~\ref{sec:related_work}).
Policy entropy, the quantity our analysis and measurements concern, is only a proxy for exploration in the state space: estimating the entropy of the visited-state distribution (App.~\ref{sec:app:experiments:coverage}), we find no consistent improvement of ReMAC over SAC on these dense-reward locomotion tasks.
Whether the higher entropy induced by ReMax translates into broader state-space exploration therefore remains open.
Furthermore, our results do not establish whether ReMax also supports deep exploration, which sparse-reward and hard-exploration domains require: on a sparse-reward variant of Reacher, ReMAC performs comparably to SAC (App.~\ref{sec:app:experiments:sparse}).
A well-aligned direction is to model the posterior over the Q-function \citep{osband2016deep,osband2018randomized}, which makes explicit the return uncertainty that ReMax already considers and may enable the deeper, structured exploration suggested by \citet{nishimori2026emergence} and \citet{tong2026finite}.
We hope our analysis of ReMax's gradient properties provides a foundation for future work.

\section*{Author Contributions}
\textbf{Soichiro Nishimori} led the research project, conducted all experiments, performed the theoretical analysis, created the visualizations, co-conceptualized the idea and wrote the manuscript.

\textbf{Paavo Parmas} mentored Soichiro Nishimori, provided advice on the theoretical analysis, algorithm design, and experimental design, co-conceptualized the idea and revised the manuscript.

\section*{Acknowledgements}
SN was supported by JSPS KAKENHI Grant Number JP24KJ0818.
PP was supported by JST ACT-X, Japan, Grant Number JPMJAX23CO.

%%%%%%%%%%%%%%%%%%%%%%%%%%%%%%%%%%%%%%%%%%%%%%%%%%%%%%%%%%%%%%%%
%% Bibliography
%%%%%%%%%%%%%%%%%%%%%%%%%%%%%%%%%%%%%%%%%%%%%%%%%%%%%%%%%%%%%%%%
\bibliography{main}
\bibliographystyle{tmlr}

\appendix

\clearpage
\startcontents[appendix]
\section*{Appendix Contents}
\printcontents[appendix]{}{1}{\setcounter{tocdepth}{2}}

\section{Proofs of Propositions in Sec.~\ref{sec:grad_analysis}}
\label{sec:app:proofs_of_propositions_in_sec_grad_analysis}

In this section, we prove the propositions in Sec.~\ref{sec:grad_analysis}.
Throughout this section, let
\[
    A_m=\mu+\sigma\xi_m,
    \qquad
    \xi_m\stackrel{\iid}{\sim}\mathcal N(0,I_d).
\]

\begin{lemma}[No ties from non-atomic cost values]
\label{lem:no_ties_from_non_atomicity}
Suppose Assumption~\ref{ass:non_degenerate_selection} holds.
Then, for every $\sigma>0$,
\[
    \mathbb P\left(
        \exists i\neq j
        \text{ such that }
        c(\mu+\sigma\xi_i)
        =
        c(\mu+\sigma\xi_j)
    \right)
    =
    0.
\]
\end{lemma}

\begin{proof}
Fix $\sigma>0$ and define $Y_i=c(\mu+\sigma\xi_i)$.
Then $Y_1,\ldots,Y_M$ are independent and identically distributed.
Since $\mu+\sigma\xi_i$ has a density with respect to Lebesgue measure, Assumption~\ref{ass:non_degenerate_selection} gives $\mathbb P(Y_i=t)=\mathbb P(\mu+\sigma\xi_i\in\{a:c(a)=t\})=0$ for every $t\in\mathbb R$, so each $Y_i$ has no atoms.
For any $i\neq j$,
\[
    \mathbb P(Y_i=Y_j)
    =
    \int_{\mathbb R}
    \mathbb P(Y_i=t)\,d\mathbb P_{Y_j}(t)
    =
    0.
\]
Taking a union bound over the finitely many pairs $(i,j)$ gives
\[
    \mathbb P(\exists i\neq j:Y_i=Y_j)=0.
\]
\end{proof}

\begin{lemma}[Differentiability of the objective]
\label{lem:differentiability_scale_objective}
Assume Assumptions~\ref{ass:smoothness} and~\ref{ass:non_degenerate_selection}.
For every $\sigma>0$, $J^M(\mu,\sigma)$ is differentiable with respect to $\mu$ and $\sigma$ and
\[
    \nabla_\mu J^M(\mu,\sigma)
    =
    \mathbb E\left[
        -\nabla c(A_{m^\star})
    \right],
    \qquad
    \partial_\sigma J^M(\mu,\sigma)
    =
    \mathbb E\left[
        -\nabla c(A_{m^\star})^\top\xi_{m^\star}
    \right],
\]
where $m^\star\in\arg\min_{m=1,\ldots,M}c(A_m)$.
\end{lemma}

\begin{proof}
Fix $\sigma>0$.
For a fixed realization of $\xi_{1:M}$, define
\[
    F(\mu',s)
    =
    \max_{m=1,\ldots,M}
    -c(\mu'+s\xi_m).
\]
By Lemma~\ref{lem:no_ties_from_non_atomicity}, the maximizer is almost surely unique at $(\mu',s)=(\mu,\sigma)$.
Hence, by Danskin's theorem \citep{danskin1967theory},
\[
    \nabla_{\mu'}F(\mu,\sigma)
    =
    -\nabla c(\mu+\sigma\xi_{m^\star}),
    \qquad
    \partial_sF(\mu,\sigma)
    =
    -\nabla c(\mu+\sigma\xi_{m^\star})^\top\xi_{m^\star}
\]
for almost every realization.

It remains to justify interchanging differentiation and expectation.
Choose $\delta\in(0,\sigma)$.
For every $\mu'$ with $\|\mu'-\mu\|\le\delta$, every $s\in[\sigma-\delta,\sigma+\delta]$, and every $m$, Assumption~\ref{ass:smoothness} gives
\[
    \|\nabla c(\mu'+s\xi_m)\|
    \le
    \|\nabla c(\mu)\|
    +
    L\delta+
    L(\sigma+\delta)\|\xi_m\|.
\]
Therefore each term $-c(\mu'+s\xi_m)$, and hence their maximum $F$, is Lipschitz
on this neighborhood with constant
\[
\Big(\|\nabla c(\mu)\| + L\delta + L(\sigma+\delta)\max_m\|\xi_m\|\Big)\Big(1+\max_m\|\xi_m\|\Big),
\]
which bounds the difference quotients of $F$ at $(\mu,\sigma)$ and is integrable
because $M<\infty$ and Gaussian random variables have finite moments.
Dominated convergence yields
\[
    \nabla_\mu J^M(\mu,\sigma)
    =
    \mathbb E\left[
        -\nabla c(A_{m^\star})
    \right],
    \qquad
    \partial_\sigma J^M(\mu,\sigma)
    =
    \mathbb E\left[
        -\nabla c(A_{m^\star})^\top\xi_{m^\star}
    \right].
\]
\end{proof}

\subsection{Proof of Proposition~\ref{prop:entropy_increase}}
\label{sec:app:proof:entropy_increase}

\begin{proof}
Let $g=\nabla c(\mu)$.
By assumption, $g\neq0$.
Define $B_m=g^\top\xi_m$.
Then $B_m\stackrel{d}{=}\|g\|Z_m$, where $Z_m\sim\mathcal N(0,1)$.
Since $g\neq0$, the random variables $B_1,\ldots,B_M$ are continuous.
Thus,
\[
    j^\star\in\arg\min_{m=1,\ldots,M}B_m
\]
is almost surely unique.

We first show that, for almost every realization, $m^\star=j^\star$ for all sufficiently small $\sigma>0$.
By Assumption~\ref{ass:smoothness},
\[
    c(\mu+\sigma\xi_m)
    =
    c(\mu)
    +
    \sigma g^\top\xi_m
    +
    R_m(\sigma),
\]
where
\[
    |R_m(\sigma)|
    \le
    \frac{L}{2}\sigma^2\|\xi_m\|^2.
\]
Fix a realization for which $j^\star$ is unique, and define
\[
    \Delta
    =
    \min_{m\neq j^\star}
    \left(
        B_m-B_{j^\star}
    \right)
    >
    0.
\]
For any $m\neq j^\star$,
\[
    c(\mu+\sigma\xi_m)
    -
    c(\mu+\sigma\xi_{j^\star})
    =
    \sigma(B_m-B_{j^\star})
    +
    R_m(\sigma)-R_{j^\star}(\sigma).
\]
Since
\[
    |R_m(\sigma)-R_{j^\star}(\sigma)|
    \le
    \frac{L}{2}\sigma^2
    \left(
        \|\xi_m\|^2+\|\xi_{j^\star}\|^2
    \right),
\]
there exists $\sigma_1>0$ such that, for all $\sigma\in(0,\sigma_1)$,
\[
    |R_m(\sigma)-R_{j^\star}(\sigma)|
    \le
    \frac{\sigma\Delta}{2}
    \qquad
    \text{for all }m\neq j^\star.
\]
Thus,
\[
    c(\mu+\sigma\xi_m)
    -
    c(\mu+\sigma\xi_{j^\star})
    \ge
    \frac{\sigma\Delta}{2}
    >
    0
\]
for all $m\neq j^\star$.
Hence, $m^\star=j^\star$ for all sufficiently small $\sigma>0$, almost surely.

By Lemma~\ref{lem:differentiability_scale_objective},
\[
    \partial_\sigma J^M(\mu,\sigma)
    =
    \mathbb E\left[
        -\nabla c(\mu+\sigma\xi_{m^\star})^\top
        \xi_{m^\star}
    \right].
\]
Using the eventual equality $m^\star=j^\star$ and the continuity of $\nabla c$,
\[
    -\nabla c(\mu+\sigma\xi_{m^\star})^\top
    \xi_{m^\star}
    \to
    -g^\top\xi_{j^\star}
    =
    -\min_{m=1,\ldots,M}g^\top\xi_m
\]
almost surely as $\sigma\downarrow0$.

We now justify dominated convergence.
For $0<\sigma\le1$, Assumption~\ref{ass:smoothness} gives
\[
    \|\nabla c(\mu+\sigma\xi_{m^\star})\|
    \le
    \|g\|
    +
    L\max_m\|\xi_m\|.
\]
Therefore,
\[
    \left|
    \nabla c(\mu+\sigma\xi_{m^\star})^\top
    \xi_{m^\star}
    \right|
    \le
    \left(
        \|g\|
        +
        L\max_m\|\xi_m\|
    \right)
    \max_m\|\xi_m\|.
\]
The right-hand side is integrable.
By dominated convergence,
\[
    \lim_{\sigma\downarrow0}
    \partial_\sigma J^M(\mu,\sigma)
    =
    -
    \mathbb E\left[
        \min_{m=1,\ldots,M}g^\top\xi_m
    \right].
\]
Since $g^\top\xi_m\stackrel{d}{=}\|g\|Z_m$,
\[
    -
    \mathbb E\left[
        \min_m g^\top\xi_m
    \right]
    =
    \|g\|
    \mathbb E\left[
        \max_m Z_m
    \right]
    =
    \|\nabla c(\mu)\|\beta_M,
\]
where $\beta_M=\mathbb E[\max_{m=1,\ldots,M}Z_m]$.
For $M\ge2$, $\beta_M>0$ because
\[
    \beta_M
    \ge
    \mathbb E[\max\{Z_1,Z_2\}]
    =
    \frac{1}{2}\mathbb E|Z_1-Z_2|
    >
    0.
\]
Therefore,
\[
    \lim_{\sigma\downarrow0}
    \partial_\sigma J^M(\mu,\sigma)
    =
    \|\nabla c(\mu)\|\beta_M
    >
    0.
\]
By the definition of a one-sided limit, there exists $\sigma_0>0$ such that
\[
    \partial_\sigma J^M(\mu,\sigma)>0
    \qquad
    \text{for all }\sigma\in(0,\sigma_0).
\]
\end{proof}

\subsection{Proof of Proposition~\ref{prop:variance_decrease}}
\label{sec:app:proof:variance_decrease}

\begin{proof}
For $M=1$,
\[
    J^1(\mu,\sigma)
    =
    -\mathbb E[c(\mu+\sigma\xi)].
\]
By Assumptions~\ref{ass:smoothness} and~\ref{ass:second_order_regularity}, differentiation under the expectation is justified by dominated convergence, and
\[
    \partial_\sigma J^1(\mu,\sigma)
    =
    -
    \mathbb E[
        \nabla c(\mu+\sigma\xi)^\top\xi
    ].
\]
Let $A=\mu+\sigma\xi$.
By Gaussian integration by parts, i.e., Stein's lemma $\mathbb E[g(\xi)\xi_i]=\mathbb E[\partial_i g(\xi)]$ \citep{stein1981estimation} applied to $g(\xi)=\partial_i c(\mu+\sigma\xi)$, for each coordinate $i$,
\[
    \mathbb E[
        \partial_i c(A)\xi_i
    ]
    =
    \sigma
    \mathbb E[
        \partial_{ii}^2 c(A)
    ].
\]
Summing over $i=1,\ldots,d$ yields
\[
    \mathbb E[
        \nabla c(A)^\top\xi
    ]
    =
    \sigma
    \mathbb E[
        \mathrm{tr}(\nabla^2 c(A))
    ].
\]
By Assumptions~\ref{ass:strong_convexity} and~\ref{ass:second_order_regularity},
\[
    \nabla^2 c(A)
    \succeq
    \lambda I_d.
\]
Therefore,
\[
    \mathrm{tr}(\nabla^2 c(A))
    \ge
    \lambda d.
\]
Hence,
\[
    \partial_\sigma J^1(\mu,\sigma)
    =
    -\sigma
    \mathbb E[
        \mathrm{tr}(\nabla^2 c(A))
    ]
    \le
    -\lambda d\sigma
    <
    0.
\]
\end{proof}

\begin{lemma}[No ties under strong convexity]
\label{lem:no_ties_strong_convex}
Assume Assumption~\ref{ass:strong_convexity}.
Then Assumption~\ref{ass:non_degenerate_selection} holds, and hence, by Lemma~\ref{lem:no_ties_from_non_atomicity}, for every $\sigma>0$,
\[
    \mathbb P\left(
        \exists i\neq j
        \text{ such that }
        c(A_i)=c(A_j)
    \right)
    =
    0.
\]
\end{lemma}

\begin{proof}
By Lemma~\ref{lem:no_ties_from_non_atomicity}, it suffices to show that every level set $L_t=\{a:c(a)=t\}$ has Lebesgue measure zero.
Fix $t\in\mathbb R$ and $y\in\mathbb R^d$, and let $e_1$ be the first standard basis vector.
By Assumption~\ref{ass:strong_convexity}, $\varphi(s)=c(y+se_1)$ is strictly convex on $\mathbb R$, so it takes the value $t$ at most twice: if $\varphi(s_1)=\varphi(s_2)=\varphi(s_3)=t$ with $s_1<s_2<s_3$, writing $s_2$ as a convex combination of $s_1$ and $s_3$ gives $\varphi(s_2)<t$, a contradiction.
Hence every line parallel to $e_1$ meets $L_t$ in at most two points.
Since $c$ is continuous, $L_t$ is closed, and Fubini's theorem shows that $L_t$ has Lebesgue measure zero.
\end{proof}

\subsection{Proof of Proposition~\ref{prop:gradient_damping}}
\label{sec:app:proof:gradient_damping}

\begin{proof}
Assumptions~\ref{ass:centered_optimum} and~\ref{ass:strong_convexity} imply
\[
    c(a)
    \ge
    \frac{\lambda}{2}\|a\|^2
    \qquad
    \text{for all }a\in\mathbb R^d.
\]
Assumptions~\ref{ass:smoothness} and~\ref{ass:centered_optimum} imply
\[
    c(a)
    \le
    \frac{L}{2}\|a\|^2
    \qquad
    \text{for all }a\in\mathbb R^d,
\]
and
\[
    \|\nabla c(a)\|
    =
    \|\nabla c(a)-\nabla c(0)\|
    \le
    L\|a\|.
\]

By Lemma~\ref{lem:no_ties_strong_convex}, Assumption~\ref{ass:non_degenerate_selection} holds.
Hence, by Lemma~\ref{lem:differentiability_scale_objective}, the gradients exist and
\[
    \nabla_\mu J^M(\mu,\sigma)
    =
    \mathbb E[-\nabla c(A_{m^\star})],
    \qquad
    \partial_\sigma J^M(\mu,\sigma)
    =
    \mathbb E[-\nabla c(A_{m^\star})^\top\xi_{m^\star}].
\]

Using $\|\nabla c(a)\|\le L\|a\|$,
\[
    \|\nabla_\mu J^M(\mu,\sigma)\|
    \le
    \mathbb E[\|\nabla c(A_{m^\star})\|]
    \le
    L\mathbb E[\|A_{m^\star}\|].
\]
Similarly,
\[
    |\partial_\sigma J^M(\mu,\sigma)|
    \le
    \mathbb E[
        \|\nabla c(A_{m^\star})\|
        \|\xi_{m^\star}\|
    ]
    \le
    L\mathbb E[
        \|A_{m^\star}\|
        \|\xi_{m^\star}\|
    ].
\]

Since $m^\star$ minimizes $c(A_m)$,
\[
    c(A_{m^\star}) = \min_{m=1,\ldots,M} c(A_m).
\]
Using the quadratic bounds on $c$,
\[
    \frac{\lambda}{2}\|A_{m^\star}\|^2
    \le
        c(A_{m^\star})
    =
    \min_m c(A_m)
    \le
    \frac{L}{2}\min_m\|A_m\|^2.
\]
Therefore,
\[
    \|A_{m^\star}\|
    \le
    \sqrt{\frac{L}{\lambda}}
    \min_{m=1,\ldots,M}\|A_m\|.
\]
Substituting this bound gives
\begin{equation}\label{eq:app:damping_intermediate}
    \|\nabla_\mu J^M(\mu,\sigma)\|
    \le
    L\sqrt{\frac{L}{\lambda}}\,
    \mathbb E\left[
        \min_m\|A_m\|
    \right],
    \qquad
    |\partial_\sigma J^M(\mu,\sigma)|
    \le
    L\sqrt{\frac{L}{\lambda}}\,
    \mathbb E\left[
        \min_m\|A_m\|
        \|\xi_{m^\star}\|
    \right].
\end{equation}
\end{proof}

\subsection{Explicit form of the damping}
\label{sec:app:proof:damping_explicit}

Proposition~\ref{prop:gradient_damping} bounds the gradients by two expectations over the $M$ samples.
The two expectations can be bounded in terms of $\mu$, $\sigma$, $d$, and $M$, which makes the dependence on the retry budget explicit.
Define the constants
$\kappa_{d,M}:=\E{\min_{m\le M}\|\xi_m\|}$,
$\tau_{d,M}:=\E{\min_{m\le M}\|\xi_m\|^2}$, and
$\nu_{d,M}:=(\E{\max_{m\le M}\|\xi_m\|^2})^{1/2}$,
which depend only on $d$ and $M$.

\begin{corollary}[Explicit form of the damping]\label{cor:damping_explicit}
Under the assumptions of Proposition~\ref{prop:gradient_damping},
\[
    \|\nabla_\mu J^M(\mu,\sigma)\|
    \le
    L\sqrt{\tfrac{L}{\lambda}}\left(\|\mu\|+\sigma\,\kappa_{d,M}\right),
    \qquad
    |\partial_\sigma J^M(\mu,\sigma)|
    \le
    L\sqrt{\tfrac{L}{\lambda}}\left(\|\mu\|+\sigma\sqrt{\tau_{d,M}}\right)\nu_{d,M}.
\]
Moreover $\kappa_{d,M}$ and $\tau_{d,M}$ are strictly decreasing in $M$, and
\[
    \kappa_{d,M}=O\!\left(\sqrt d\,M^{-1/d}\right),
    \qquad
    \tau_{d,M}=O\!\left(d\,M^{-2/d}\right),
    \qquad
    \nu_{d,M}=O\!\left(\sqrt{d+\log M}\right),
\]
where the constants in $O(\cdot)$ depend on neither $d$ nor $M$.
\end{corollary}

\paragraph{Notes on the bounds.}
\begin{itemize}
\item \textbf{Near the optimum:}
Both bounds increase in $\|\mu\|$ and $\sigma$, so the gradient norm converges to zero as $(\mu,\sigma)\to(0,0)$.
This is a limiting statement; it does not imply that the gradient norm decreases monotonically along an optimization trajectory.
\item \textbf{Mean gradient:}
At any fixed $(\mu,\sigma)$, the bound decreases in $M$ through $\kappa_{d,M}$.
It is an $M$-decreasing upper envelope for the mean-gradient norm, not a statement that the norm itself decreases with $M$.
\item \textbf{Scale gradient:}
The bound contains $\sqrt{\tau_{d,M}}$, which decreases with $M$, and $\nu_{d,M}$, which increases with $M$, so it need not be monotone in $M$.
At $\mu=0$, however, a bound with a faster rate in $M$ holds: since $\xi_{m^\star}=A_{m^\star}/\sigma$, the argument in the proof of Proposition~\ref{prop:gradient_damping} gives $|\partial_\sigma J^M(0,\sigma)|\le(L^2/\lambda)\,\sigma\,\tau_{d,M}=O(\sigma\,d\,M^{-2/d})$, which decreases in $M$.
\end{itemize}

\begin{proof}
We bound the two expectations of Eq.~\eqref{eq:app:damping_intermediate} in turn.

\emph{Step 1: the mean gradient.}
For every $m$, the triangle inequality gives $\|A_m\|\le\|\mu\|+\sigma\|\xi_m\|$.
Evaluating the right-hand side at an index attaining $\min_m\|\xi_m\|$, and using that a minimum is no larger than the value at any particular index,
\begin{equation}\label{eq:app:min_norm_bound}
    \min_{m\le M}\|A_m\|
    \;\le\;
    \|\mu\|+\sigma\min_{m\le M}\|\xi_m\|
    \qquad\text{almost surely.}
\end{equation}
Taking expectations, $\mathbb E[\min_m\|A_m\|]\le\|\mu\|+\sigma\kappa_{d,M}$, which with Eq.~\eqref{eq:app:damping_intermediate} gives the first bound.

\emph{Step 2: the scale gradient.}
Since $m^\star$ is one of the $M$ indices, $\|\xi_{m^\star}\|\le\max_{m\le M}\|\xi_m\|$ almost surely, so by Cauchy--Schwarz
\[
    \mathbb E\left[\min_m\|A_m\|\,\|\xi_{m^\star}\|\right]
    \le
    \left(\mathbb E\left[\left(\min_m\|A_m\|\right)^2\right]\right)^{1/2}
    \left(\mathbb E\left[\max_m\|\xi_m\|^2\right]\right)^{1/2}
    =
    \left(\mathbb E\left[\left(\min_m\|A_m\|\right)^2\right]\right)^{1/2}\nu_{d,M}.
\]
Squaring Eq.~\eqref{eq:app:min_norm_bound} and taking expectations,
\[
    \mathbb E\left[\left(\min_m\|A_m\|\right)^2\right]
    \le
    \|\mu\|^2+2\|\mu\|\sigma\kappa_{d,M}+\sigma^2\tau_{d,M}
    \le
    \left(\|\mu\|+\sigma\sqrt{\tau_{d,M}}\right)^2,
\]
where the last step uses $\kappa_{d,M}\le\sqrt{\tau_{d,M}}$, which is Jensen's inequality applied to $x\mapsto x^2$.
Combining the two displays gives the second bound.

\emph{Step 3: monotonicity in $M$.}
Write $X_M=\min_{m\le M}\|\xi_m\|$.
Then $X_{M+1}=\min\{X_M,\|\xi_{M+1}\|\}\le X_M$ almost surely, and $\mathbb P(\|\xi_{M+1}\|<X_M)>0$ because $\|\xi\|$ has full support on $(0,\infty)$.
Hence $\mathbb E[X_{M+1}]<\mathbb E[X_M]$, so $\kappa_{d,M}$ is strictly decreasing in $M$; the same argument applied to $X_M^2$ gives strict monotonicity of $\tau_{d,M}$.

\emph{Step 4: rates in $M$ and $d$.}
Let $\chi_m=\|\xi_m\|$, which is chi-distributed with $d$ degrees of freedom, with distribution function $F$; by independence, $\mathbb P(X_M>t)=(1-F(t))^M$.
We bound this by splitting $t$ at $\sqrt{2d}$: a small-ball bound for $t\le\sqrt{2d}$ and Markov's inequality for $t\ge\sqrt{2d}$.
Let $b=e\sqrt{\pi/2}$.
For $t\in[0,\sqrt{2d}]$, since $\|\xi_1\|\le\sqrt d\max_{i\le d}|(\xi_1)_i|$ and the standard normal density is at least $e^{-1}/\sqrt{2\pi}$ on $[-\sqrt2,\sqrt2]$,
\[
    F(t)
    \ge
    \mathbb P\left(\max_{i\le d}|(\xi_1)_i|\le\frac{t}{\sqrt d}\right)
    \ge
    \left(\frac{2t}{\sqrt d}\cdot\frac{e^{-1}}{\sqrt{2\pi}}\right)^d
    =
    \left(\frac{t}{b\sqrt d}\right)^d.
\]
Moreover, Markov's inequality gives $\mathbb P(\chi_1^2\ge2d)\le\mathbb E[\chi_1^2]/(2d)=\tfrac12$, so $1-F(t)\le\tfrac12$ for $t\ge\sqrt{2d}$.
Using $1-x\le e^{-x}$ for the first and $1-F(t)\le\tfrac12$ for the second, we have the two estimates
\[
    (1-F(t))^M
    \le
    \begin{cases}
        \exp\bigl(-M\,(t/(b\sqrt d))^d\bigr), & t\in[0,\sqrt{2d}],\\
        2^{1-M}\,\mathbb P(\chi_1>t), & t\ge\sqrt{2d}.
    \end{cases}
\]

\emph{Bound on $\kappa_{d,M}$.}
By the layer-cake representation, $\kappa_{d,M}=\int_0^\infty(1-F(t))^M\,dt$, and the two estimates give
\[
    \kappa_{d,M}
    \le
    \int_0^\infty e^{-M(t/(b\sqrt d))^d}\,dt
    +
    2^{1-M}\int_0^\infty\mathbb P(\chi_1>t)\,dt
    =
    \Gamma\!\left(1+\tfrac1d\right)b\sqrt d\,M^{-1/d}
    +
    2^{1-M}\,\mathbb E[\chi_1],
\]
where, for the first term, substituting $t=b\sqrt d\,M^{-1/d}u^{1/d}$ pulls $M$ out of the integral and leaves $\int_0^\infty e^{-u}\,d(u^{1/d})=\Gamma(1+\tfrac1d)$.
Since $\Gamma(1+\tfrac1d)\le1$, $\mathbb E[\chi_1]\le(\mathbb E[\chi_1^2])^{1/2}=\sqrt d$, and $2^{1-M}\le M^{-1}\le M^{-1/d}$, we obtain $\kappa_{d,M}\le(b+1)\sqrt d\,M^{-1/d}$.

\emph{Bound on $\tau_{d,M}$.}
Likewise, $\tau_{d,M}=\int_0^\infty 2t\,(1-F(t))^M\,dt$, and the two estimates give
\[
    \tau_{d,M}
    \le
    \int_0^\infty 2t\,e^{-M(t/(b\sqrt d))^d}\,dt
    +
    2^{1-M}\int_0^\infty 2t\,\mathbb P(\chi_1>t)\,dt
    =
    \Gamma\!\left(1+\tfrac2d\right)b^2d\,M^{-2/d}
    +
    2^{1-M}\,\mathbb E[\chi_1^2],
\]
where, for the first term, the same substitution leaves $\int_0^\infty e^{-u}\,d(u^{2/d})=\Gamma(1+\tfrac2d)$.
Since $\Gamma(1+\tfrac2d)\le2$, $\mathbb E[\chi_1^2]=d$, and $2^{1-M}\le\tfrac94M^{-2}\le\tfrac94M^{-2/d}$ (as $\tfrac94=\max_{M\ge1}2^{1-M}M^2$, attained at $M=3$), we obtain $\tau_{d,M}\le(2b^2+\tfrac94)\,d\,M^{-2/d}$.

\emph{Bound on $\nu_{d,M}$.}
Since $\chi_m^2$ is chi-squared with $d$ degrees of freedom, its moment generating function gives $\mathbb E[e^{\chi_m^2/4}]=(1-2\cdot\tfrac14)^{-d/2}=2^{d/2}$.
Hence, by Jensen's inequality and $\max_m e^{\chi_m^2/4}\le\sum_m e^{\chi_m^2/4}$,
\[
    \nu_{d,M}^2
    =
    \mathbb E\left[\max_{m\le M}\chi_m^2\right]
    \le
    4\log\mathbb E\left[\max_{m\le M}e^{\chi_m^2/4}\right]
    \le
    4\log\left(M\,2^{d/2}\right)
    =
    4\log M+2d\log2
    \le
    4(d+\log M),
\]
so $\nu_{d,M}\le2\sqrt{d+\log M}$.
\end{proof}

\section{Details of the gradient analysis}\label{sec:app:gradient}
We provide additional details for the vector field analysis in Sec.~\ref{sec:grad_analysis:vector_field} and the effect of Adam's normalization in Sec.~\ref{sec:grad_analysis:adam_effect}.

\subsection{Details for the vector field analysis}\label{sec:app:gradient:vector_field}
In Sec.~\ref{sec:grad_analysis:vector_field}, we visualized the gradient landscape of the ReMax objective to build intuition for its exploration-encouraging properties.
Here we describe the setup and methodology used to generate these vector fields.

We consider a stateless continuous-control problem in which the policy is a one-dimensional Gaussian distribution $\mathcal{N}(\mu,\sigma^2)$.
To enable differentiation through sampling, actions are parameterized using the reparameterization trick, $a_i=\mu+\sigma\xi_i$, where $\xi_i\sim\mathcal{N}(0,1)$ are \iid\ standard normal noise variables.
We use the quadratic deterministic reward $r(a)=-a^2$, making $a=0$ the unique optimal action.
For entropy regularization, we set $\eta=0.1$ and use the policy-level objective $J_{\mathrm{ent}}(\mu,\sigma)=\mathbb{E}_{a\sim\mathcal{N}(\mu,\sigma^2)}[r(a)]+\eta\,\mathcal{H}(\mathcal{N}(\mu,\sigma^2))$, where $\mathcal{H}(\mathcal{N}(\mu,\sigma^2))=\frac12\log(2\pi e\sigma^2)$.

To construct the vector field, we evaluate gradients of the ReMax objective with respect to $\mu$ and $\sigma$ on a grid defined by $\mu\in[-2,2]$ and $\sigma\in[0.1,2]$.
For each coordinate $(\mu,\sigma)$, we sample a batch of $B=16$ actions and compute gradients using the reparameterization (RP) estimator with automatic differentiation.
A single stochastic gradient evaluation would produce a noisy vector field.
To approximate the expected gradient landscape $\nabla_{\mu,\sigma}\mathcal{J}^M$, we therefore average RP gradients over 100 independent Monte Carlo trials for each grid coordinate.

\subsection{Details for Adam's normalization}\label{sec:app:gradient:adam_effect}
In Sec.~\ref{sec:grad_analysis:adam_effect}, we showed that Adam's normalization can mitigate the damping effect of ReMax.
To illustrate this, we plot the convergence path of $(\mu,\sigma)$ in the two-dimensional parameter space in Fig.~\ref{fig:adam_path}.
As $M$ increases, the trajectory becomes more gradual around the optimum, consistent with the vector-field analysis.

Furthermore, the consistency with SGD observed for larger $\epsilon$ in Fig.~\ref{fig:adam_effect} also appears in these trajectories, supporting the argument that increasing $\epsilon$ makes the optimization path follow the vector field more closely.

\begin{figure}[t]
    \centering
    \includegraphics[width=1.0\textwidth]{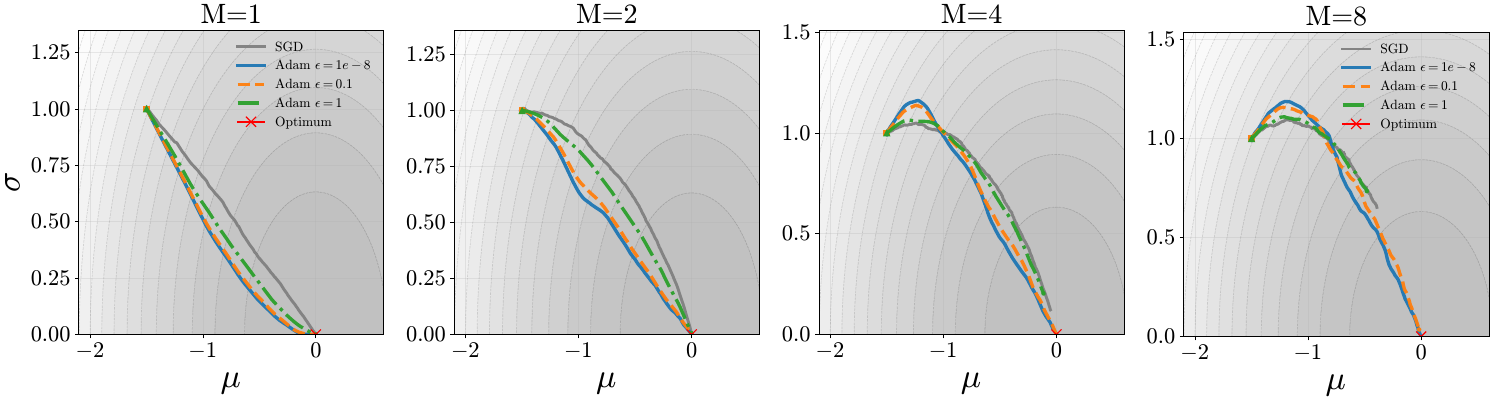}
    \caption{
        The two-dimensional trajectories of the parameter $(\mu, \sigma)$ optimized by Adam starting from $(\mu, \sigma) = (-1.5, 1.0)$ with different $\epsilon$ values ($\epsilon \in \{10^{-8}, 10^{-1}, 1.0\}$) and by SGD.
        As the vector field indicates, increasing $M$ makes the path deviate more from the optimum and slows convergence.
        However, with smaller $\epsilon$, convergence near the optimum remains fast, whereas, for $M\ge4$, larger $\epsilon$ aligns the trajectory with the SGD path that follows the vector field.
    }
    \label{fig:adam_path}
\end{figure}

\paragraph{Isolating the denominator.}
Changing $\epsilon$ rescales the adaptive normalization but never removes it.
To isolate the denominator we therefore compare an \emph{EMA-only} variant, $\theta_{t+1}=\theta_t-\alpha\hat m_t$, which removes the entire adaptive second-moment denominator while retaining Adam's first-moment EMA and its bias correction.
Fig.~\ref{fig:adam_denominator} compares it with Adam at $\epsilon\in\{10^{-8},10^{-1},1\}$ and with SGD on the same toy problem ($\alpha=0.01$, $B=64$, $T=1000$, 5 seeds).
EMA-only retains the full $M$-dependent slowdown and reproduces SGD's behavior well, whereas Adam with $\epsilon=10^{-8}$ removes most of it.
The first-moment bias correction is retained in EMA-only, yet the full $M$-dependent slowdown remains.
This supports the interpretation that the adaptive second-moment denominator, rather than the first-moment bias correction, is responsible for removing much of the gradient-magnitude effect.

\begin{figure}[h]
    \centering
    \includegraphics[width=1.0\textwidth]{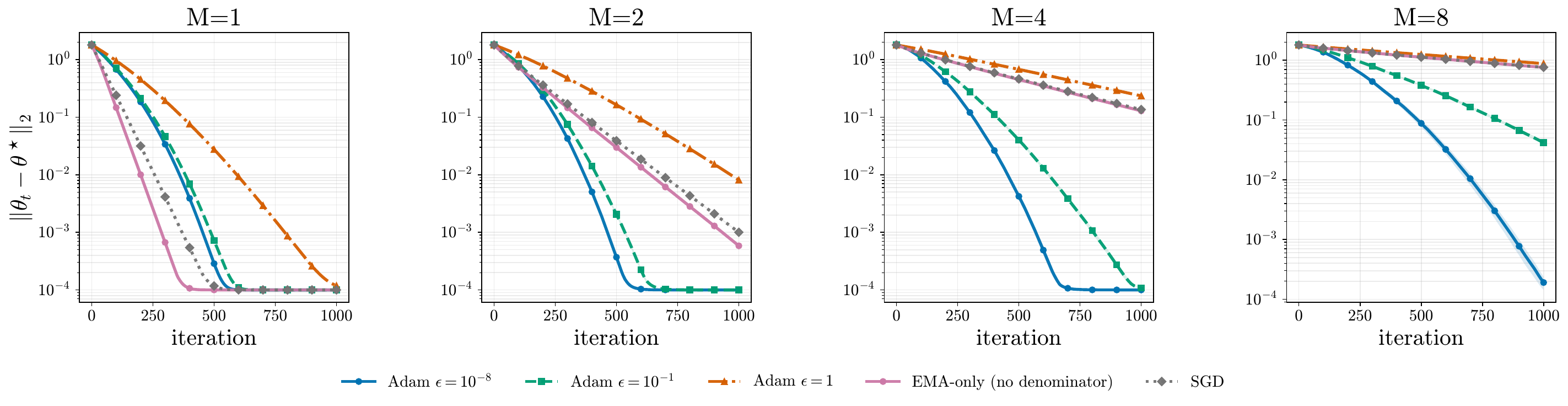}
    \caption{
        Isolating Adam's denominator.
        Distance to the optimum $(\mu,\sigma)=(0,0)$ for Adam ($\epsilon\in\{10^{-8},10^{-1},1\}$), for EMA-only (Adam with the denominator removed, $\theta_{t+1}=\theta_t-\alpha\hat m_t$), and for SGD; shading is the standard error over 5 seeds.
        EMA-only preserves the $M$-dependent slowdown and, for $M\ge4$, coincides with SGD and with large-$\epsilon$ Adam, whereas $\epsilon=10^{-8}$ largely removes it.
    }
    \label{fig:adam_denominator}
\end{figure}

\section{Details of the experimental setting}\label{sec:app:experiments}\label{sec:app:experiments:setup}
Here we provide additional details for the experiments in Sec.~\ref{sec:experiments}.

\paragraph{Implementation.}
The implementation of the ReMAC algorithm is based on the \texttt{rejax}\footnote{\url{https://github.com/keraJLi/rejax}} implementation of SAC.
The code is available at \url{https://github.com/nissymori/remax_continuous}.

\paragraph{Statistics.}
Unless stated otherwise, the \emph{final} value of a run is the mean over the last $10\%$ of training, which is robust to evaluation noise, and we report the mean $\pm$ the standard error across the 10 seeds.

\paragraph{Hyperparameters.}
We follow the SAC and PPO implementations and hyperparameters provided in the \texttt{rejax} library \citep{liesen2024rejax}, which contains environment-specific tuned configurations for Brax tasks.
Listing all hyperparameters would largely duplicate that repository, so we refer readers to the environment-specific configuration files in our repository at \texttt{brax/configs/} for the exact configurations used in our experiments.
The hyperparameters (learning rate, batch size, discount factor, Polyak factor, etc.) were tuned for each environment by the authors of \texttt{rejax}.
We tuned the learning rate $\alpha$, shared by the actor and the critic, for ReMAC by sweeping $\{10^{-4},2\times10^{-4},3\times10^{-4},5\times10^{-4},10^{-3}\}$ with 3 random seeds.
The selected learning rates are provided in Tab.~\ref{tab:hyperparameters}.

\begin{table}[t]
    \centering
    \caption{Selected hyperparameters for ReMAC.}
    \label{tab:hyperparameters}
    \begin{tabular}{c|c}
        \textbf{Environment} & \textbf{Learning Rate}
        \\
        HalfCheetah & $2\times10^{-4}$ \\
        Ant & $5\times10^{-4}$ \\
        Hopper & $3\times10^{-4}$ \\
        Reacher & $5\times10^{-4}$ \\
        Swimmer & $10^{-4}$ \\
        Walker2d & $10^{-4}$ \\
        HumanoidStandup & $10^{-4}$ \\
    \end{tabular}
\end{table}

\section{Additional results}\label{sec:app:experiments:additional_results}
Here we provide additional results for the experiments in Sec.~\ref{sec:experiments}.

\subsection{Entropy and return with different $\epsilon$}\label{sec:app:experiments:additional_results:entropy}
We plot policy entropy and return for each $\epsilon\in\{10^{-8},10^{-2},10^{-1},1\}$ in Figs.~\ref{fig:entropy_vary_m_eps_1em8}--\ref{fig:return_vary_m_eps_1}, with $M\in\{1,2,4,8\}$ overlaid in each panel.
Within a panel, $M$ carries the color that Figs.~\ref{fig:main_result_return} and~\ref{fig:main_entropy} assign to it, together with a distinct line style and marker.
For every $M$, larger $\epsilon$ generally yields higher entropy, although the trend is not monotone in every environment.
On HumanoidStandup, the one task where $M>1$ does not clearly raise entropy at the default $\epsilon$, every $M>1$ exceeds $M=1$ at $\epsilon=1$.
This is consistent with Sec.~\ref{sec:grad_analysis:adam_effect}, while the return tends to decrease for larger $\epsilon$, because a larger $\epsilon$ reduces the effective step size and can slow learning.

\subsection{Ablation by the learning rate}\label{sec:app:experiments:lr}
Since Adam's $\epsilon$ rescales the per-parameter step size, a natural question is whether its effect can be reproduced simply by changing the global learning rate $\alpha$.
To test this, we run ReMAC with $M=4$ on HalfCheetah, Swimmer, and Walker2d, sweeping $\alpha$ over $\{10^{-4},3\times10^{-4},5\times10^{-4},10^{-3}\}$ for $\epsilon=1$ and over $\{10^{-5},3\times10^{-5},5\times10^{-5},10^{-4}\}$ for the default $\epsilon=10^{-8}$, where $\alpha=10^{-4}$, shared by both sweeps, serves as the reference.
Figs.~\ref{fig:lr_sweep_return} and~\ref{fig:lr_sweep_entropy} plot the resulting return and entropy.
At $\epsilon=1$, where the large denominator strongly damps the effective step size, increasing $\alpha$ can recover part of the return lost relative to $\epsilon=10^{-8}$, while the entropy decreases as $\alpha$ grows; in HalfCheetah, it stays above the $\epsilon=10^{-8}$ level.
At $\epsilon=10^{-8}$, decreasing $\alpha$ below $10^{-4}$ mainly slows learning, and does not raise the entropy of HalfCheetah to the level attained at $\epsilon=1$.
Hence $\alpha$ and $\epsilon$ are not interchangeable: $\alpha$ scales the whole update uniformly, whereas $\epsilon$ reshapes the per-coordinate normalization and thereby controls the entropy, consistent with the analysis in Sec.~\ref{sec:grad_analysis:adam_effect}.
The best pair of $\epsilon$ and $\alpha$ depends on the environment. This suggests tuning them jointly.

\subsection{Ablation by the batch size}\label{sec:app:experiments:bs}
The unbiased ReMax estimator draws $B$ action samples per state to estimate the order-$M$ objective and is only defined for $B\ge M$, with larger $B$ reducing the variance of the estimate.
We ablate $B\in\{8,16\}$ at $\epsilon=10^{-8}$ with the tuned learning rate for $M\in\{1,2,4,8\}$ in Fig.~\ref{fig:return_bs_ablation}.
In most cases the standard-error bands of $B=8$ and $B=16$ overlap, and neither is consistently better; on HalfCheetah, $B=8$ attains a higher return.

\subsection{Computational cost}\label{sec:app:experiments:cost}
We report the wall-clock training time of SAC and ReMAC with $B=8$ and $B=16$ on HalfCheetah for 3M environment steps, vectorized over 10 random seeds.
This setting provides a controlled comparison because SAC and ReMAC share the same training pipeline and network architectures and differ mainly in the actor loss, while ReMAC also omits the temperature update used for entropy regularization.
The evaluation frequency is once every 30,000 environment steps, and the training time is measured on a GeForce RTX 2080 Ti.
The results are reported in Tab.~\ref{tab:computational_cost}.
ReMAC is slower than SAC because its actor update requires additional Q-network evaluations for $B$ action samples per state.
As expected, the wall-clock time increases with $B$.
This additional computational cost is a limitation of the current implementation.

\begin{table}[h]
    \centering
    \caption{Wall-clock training time of SAC and ReMAC with $B=8$ and $B=16$ on HalfCheetah for 3M environment steps, vectorized over 10 random seeds.}
    \label{tab:computational_cost}
    \begin{tabular}{c|c|c}
        \textbf{SAC} & \textbf{ReMAC ($B=8$)} & \textbf{ReMAC ($B=16$)}
        \\
        \hline
        1 h 03 min & 1 h 33 min & 2 h 06 min
    \end{tabular}
\end{table}

\begin{figure}[t]
    \centering
    \includegraphics[width=1.0\textwidth]{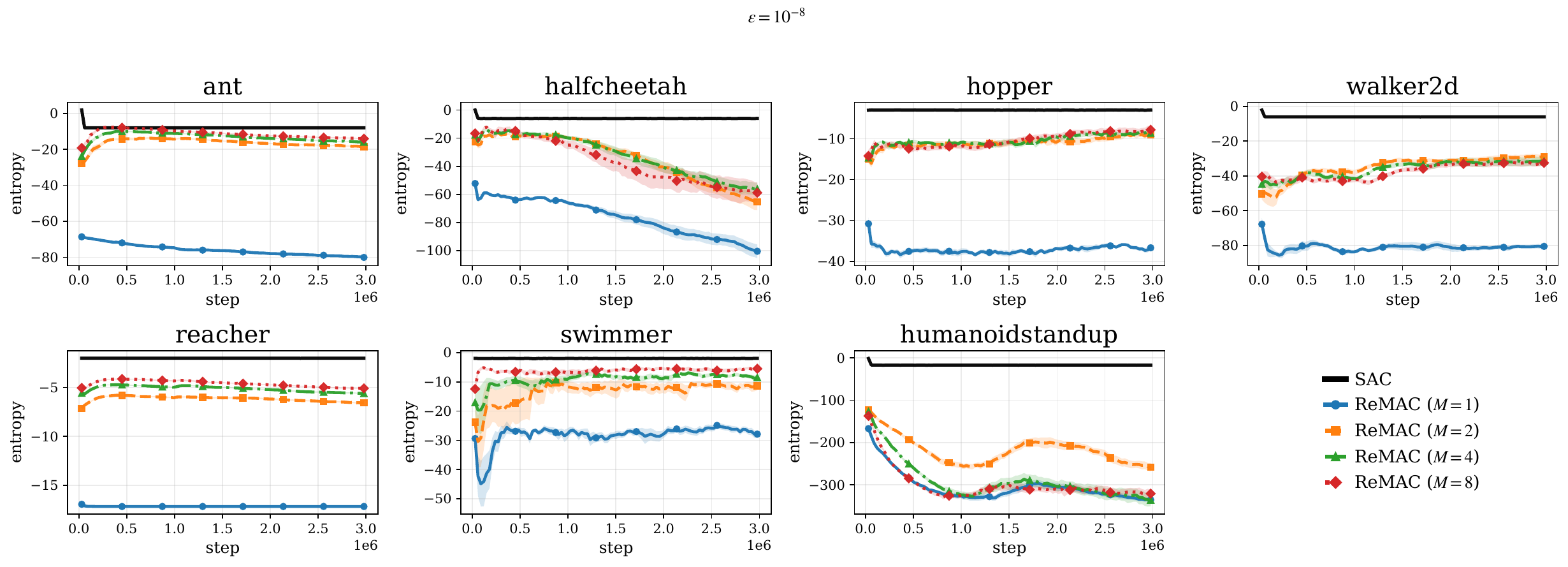}
    \caption{
        The average policy entropy of ReMAC with $M\in\{1,2,4,8\}$ at Adam's default $\epsilon=10^{-8}$, for all tasks.
        Within a panel, $M$ carries the color that Figs.~\ref{fig:main_result_return} and~\ref{fig:main_entropy} assign to it, together with a distinct line style and marker; SAC is shown for reference.
        Larger $\epsilon$ generally yields higher entropy (Figs.~\ref{fig:entropy_vary_m_eps_1em8}--\ref{fig:entropy_vary_m_eps_1}).
    }
    \label{fig:entropy_vary_m_eps_1em8}
\end{figure}

\begin{figure}[t]
    \centering
    \includegraphics[width=1.0\textwidth]{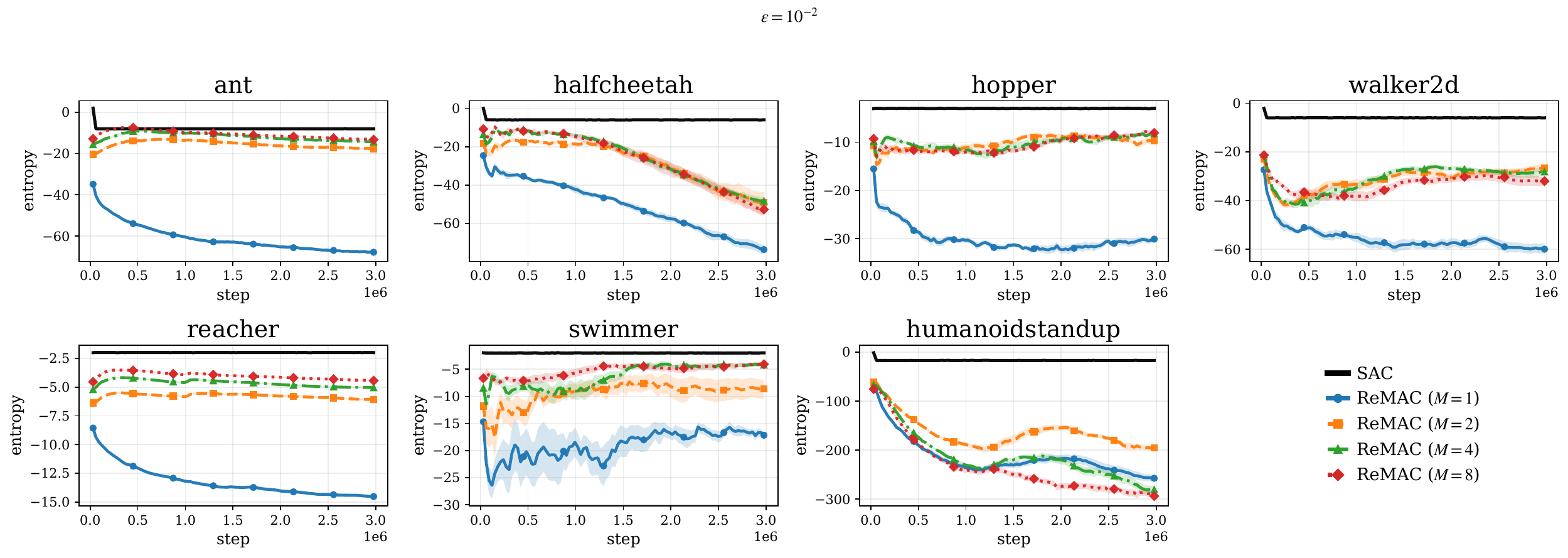}
    \caption{
        The average policy entropy of ReMAC with $M\in\{1,2,4,8\}$ at $\epsilon=10^{-2}$, for all tasks.
        Same encoding as Fig.~\ref{fig:entropy_vary_m_eps_1em8}.
    }
    \label{fig:entropy_vary_m_eps_1em2}
\end{figure}

\begin{figure}[t]
    \centering
    \includegraphics[width=1.0\textwidth]{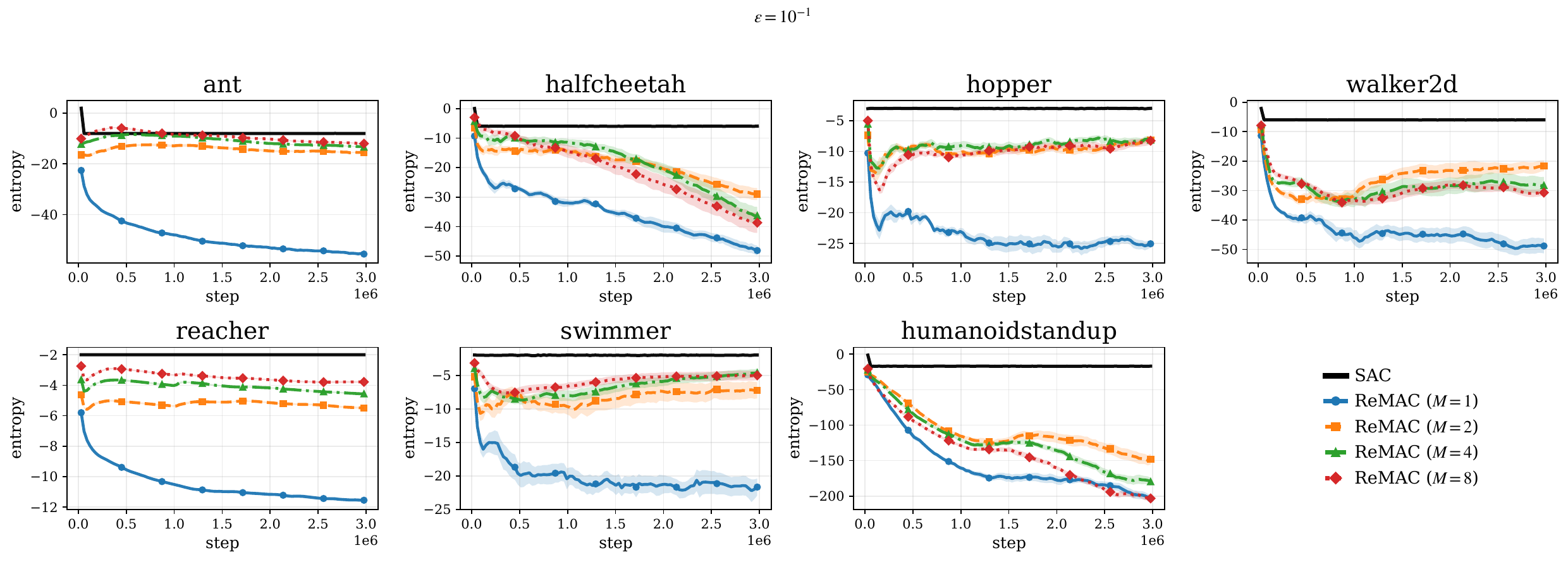}
    \caption{
        The average policy entropy of ReMAC with $M\in\{1,2,4,8\}$ at $\epsilon=10^{-1}$, for all tasks.
        Same encoding as Fig.~\ref{fig:entropy_vary_m_eps_1em8}.
    }
    \label{fig:entropy_vary_m_eps_1em1}
\end{figure}

\begin{figure}[t]
    \centering
    \includegraphics[width=1.0\textwidth]{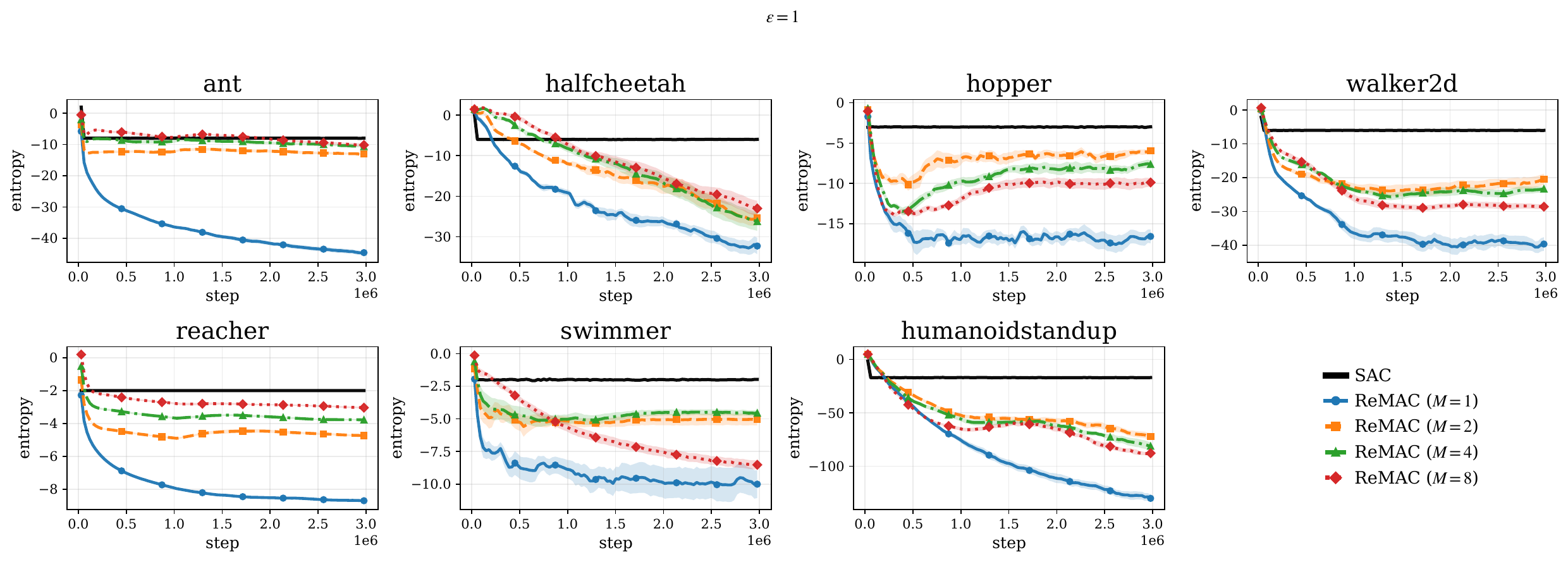}
    \caption{
        The average policy entropy of ReMAC with $M\in\{1,2,4,8\}$ at $\epsilon=1$, for all tasks.
        Same encoding as Fig.~\ref{fig:entropy_vary_m_eps_1em8}.
    }
    \label{fig:entropy_vary_m_eps_1}
\end{figure}

\begin{figure}[t]
    \centering
    \includegraphics[width=1.0\textwidth]{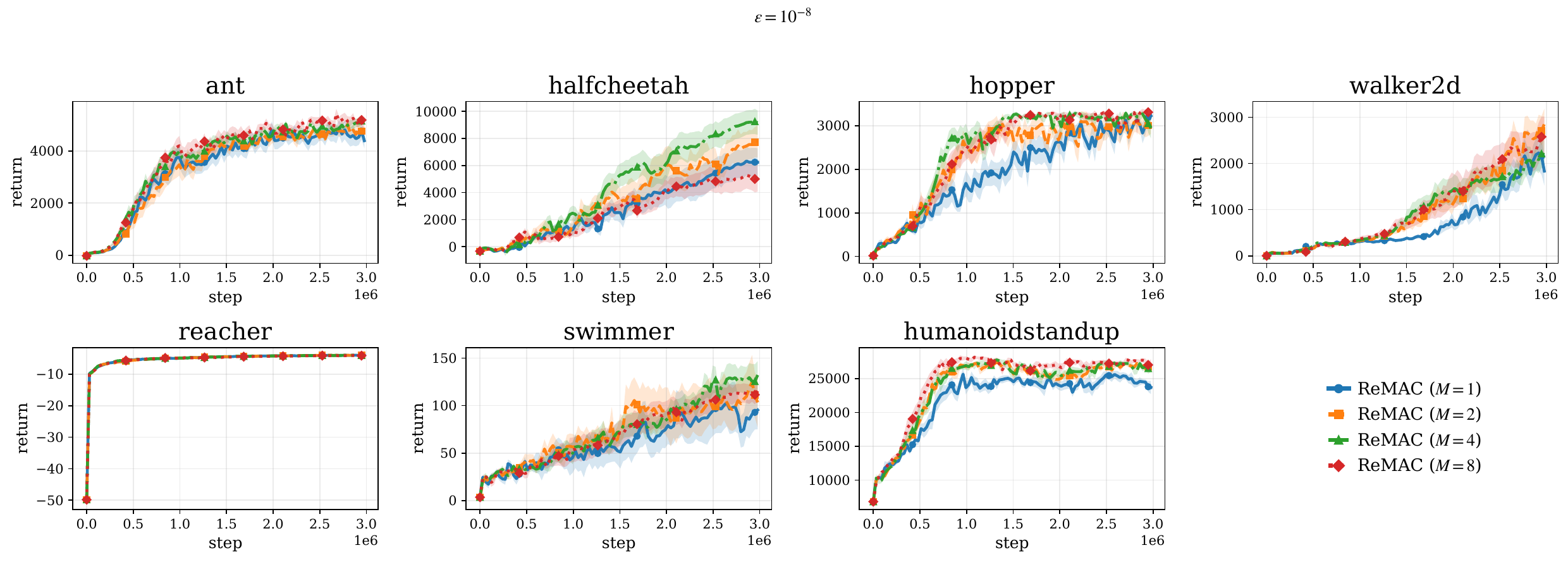}
    \caption{
        The average return of ReMAC with $M\in\{1,2,4,8\}$ at Adam's default $\epsilon=10^{-8}$, for all tasks.
        Same encoding as Fig.~\ref{fig:entropy_vary_m_eps_1em8}, without the SAC reference.
    }
    \label{fig:return_vary_m_eps_1em8}
\end{figure}

\begin{figure}[t]
    \centering
    \includegraphics[width=1.0\textwidth]{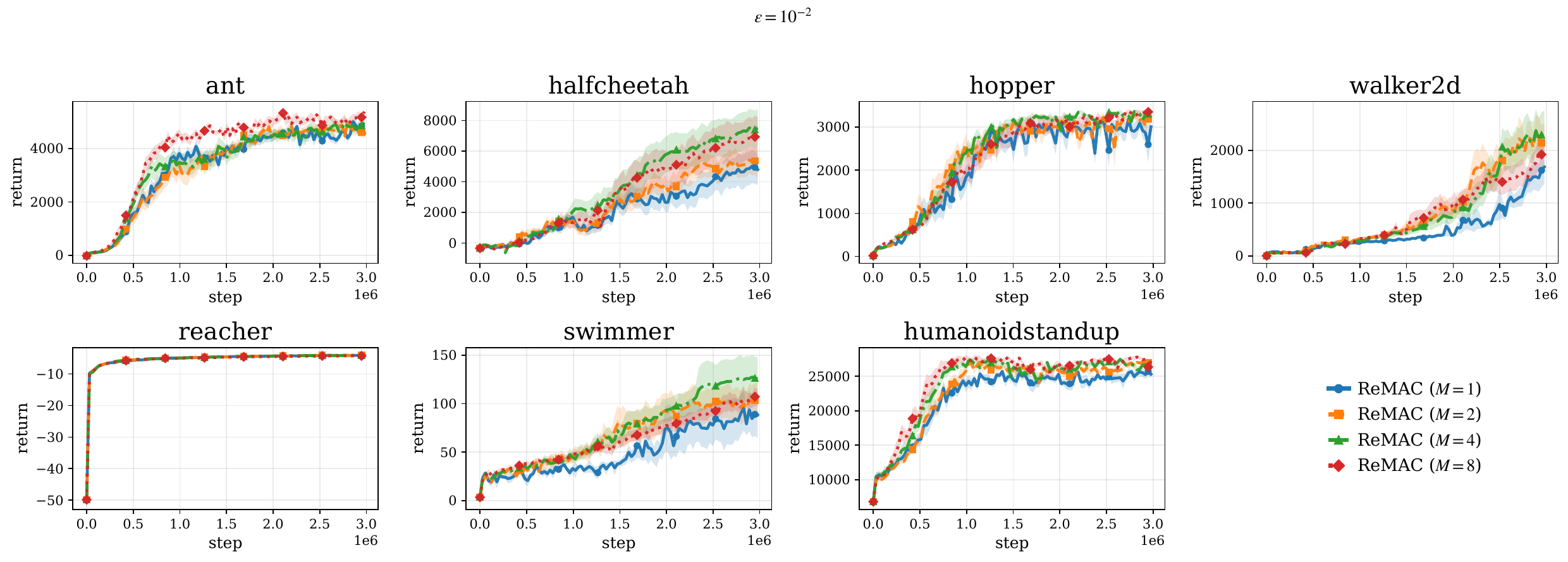}
    \caption{
        The average return of ReMAC with $M\in\{1,2,4,8\}$ at $\epsilon=10^{-2}$, for all tasks.
        Same encoding as Fig.~\ref{fig:return_vary_m_eps_1em8}.
    }
    \label{fig:return_vary_m_eps_1em2}
\end{figure}

\begin{figure}[t]
    \centering
    \includegraphics[width=1.0\textwidth]{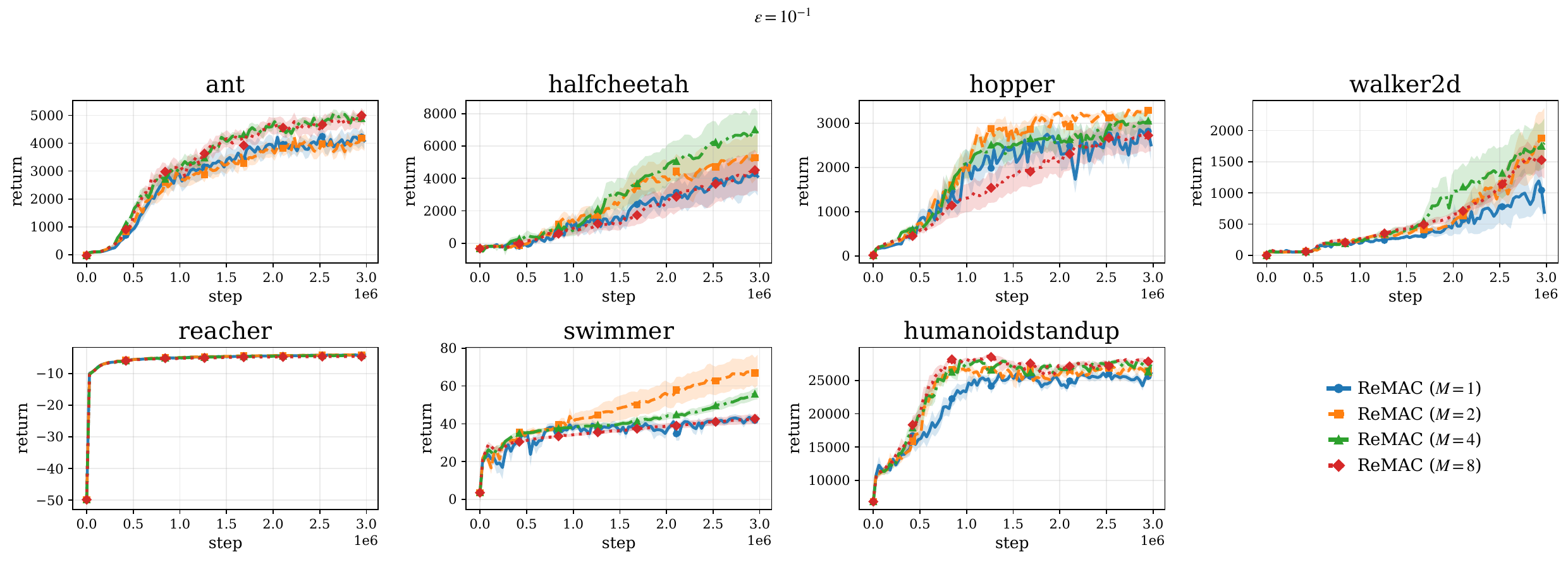}
    \caption{
        The average return of ReMAC with $M\in\{1,2,4,8\}$ at $\epsilon=10^{-1}$, for all tasks.
        Same encoding as Fig.~\ref{fig:return_vary_m_eps_1em8}.
    }
    \label{fig:return_vary_m_eps_1em1}
\end{figure}

\begin{figure}[t]
    \centering
    \includegraphics[width=1.0\textwidth]{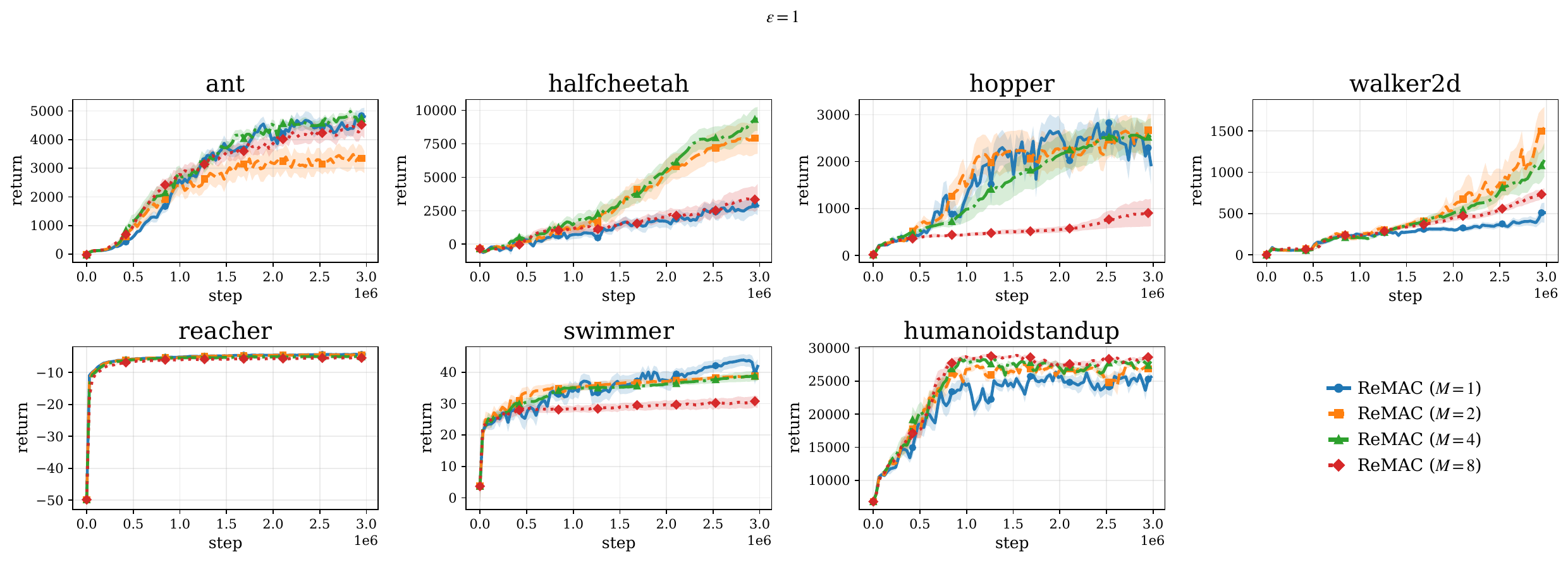}
    \caption{
        The average return of ReMAC with $M\in\{1,2,4,8\}$ at $\epsilon=1$, for all tasks.
        Same encoding as Fig.~\ref{fig:return_vary_m_eps_1em8}.
    }
    \label{fig:return_vary_m_eps_1}
\end{figure}

\begin{figure}[t]
    \centering
    \includegraphics[width=1.0\textwidth]{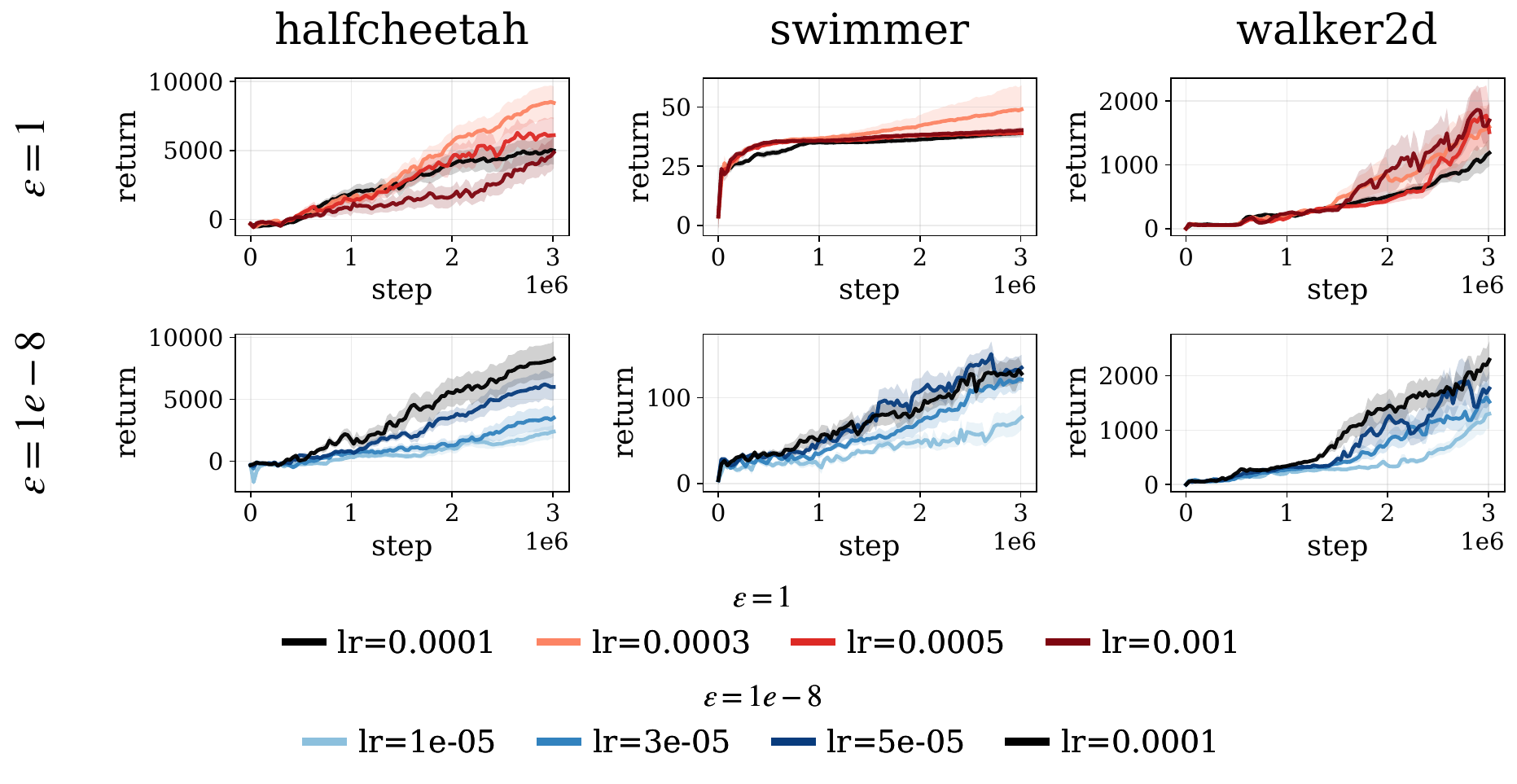}
    \caption{
        The average return of ReMAC with $M=4$ under different learning rates:
        $\epsilon=1$ (top row, $\alpha\in\{10^{-4},3\times10^{-4},5\times10^{-4},10^{-3}\}$) and
        $\epsilon=10^{-8}$ (bottom row, $\alpha\in\{10^{-5},3\times10^{-5},5\times10^{-5},10^{-4}\}$),
        for HalfCheetah, Swimmer, and Walker2d.
        The shared reference learning rate $\alpha=10^{-4}$ is drawn in black in both rows.
    }
    \label{fig:lr_sweep_return}
\end{figure}

\begin{figure}[t]
    \centering
    \includegraphics[width=1.0\textwidth]{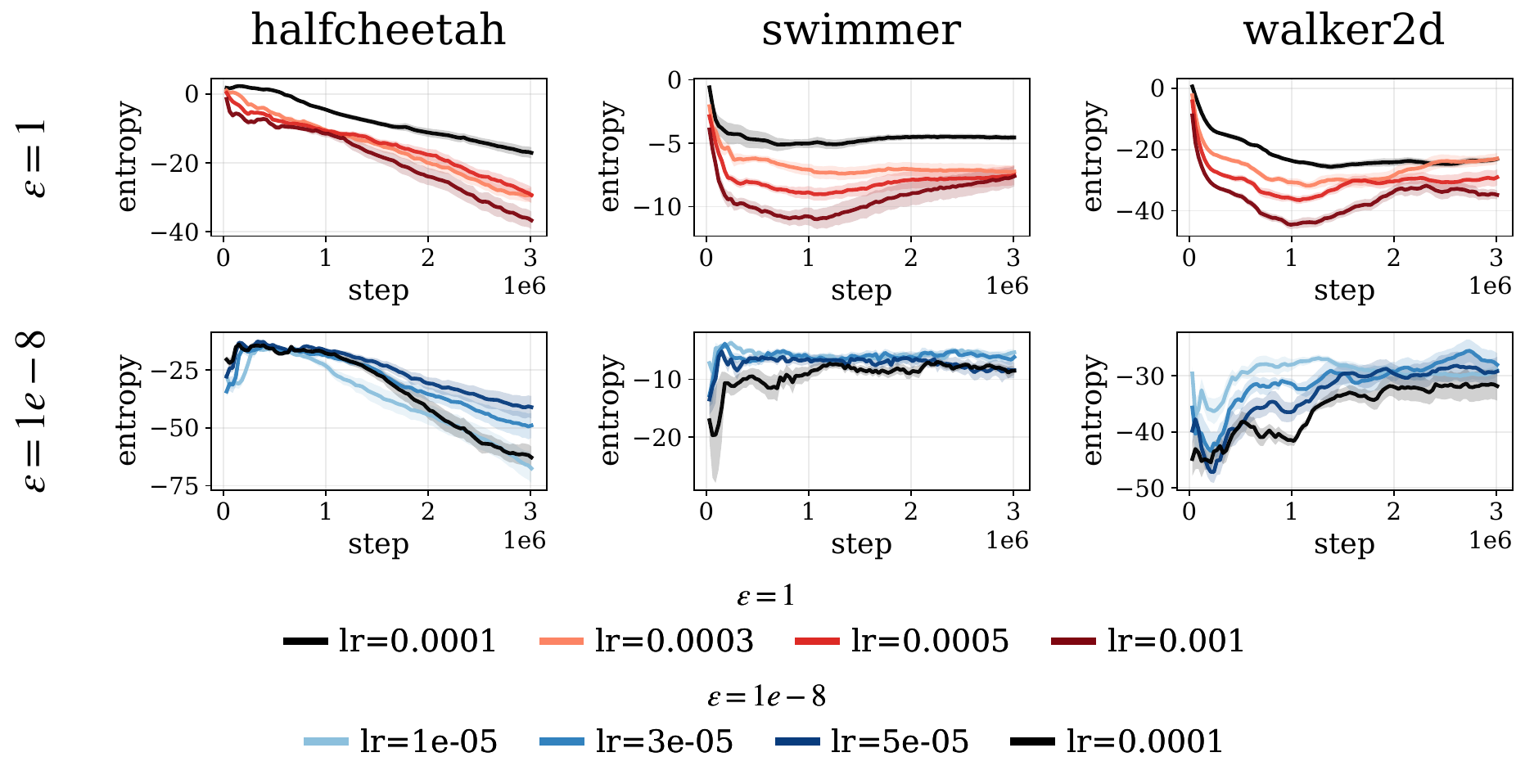}
    \caption{
        The average policy entropy of ReMAC with $M=4$ under different learning rates:
        $\epsilon=1$ (top row, $\alpha\in\{10^{-4},3\times10^{-4},5\times10^{-4},10^{-3}\}$) and
        $\epsilon=10^{-8}$ (bottom row, $\alpha\in\{10^{-5},3\times10^{-5},5\times10^{-5},10^{-4}\}$),
        for HalfCheetah, Swimmer, and Walker2d.
        The shared reference learning rate $\alpha=10^{-4}$ is drawn in black in both rows.
    }
\label{fig:lr_sweep_entropy}

\end{figure}

\begin{figure}[t]
    \centering
    \begin{subfigure}[t]{0.49\textwidth}
        \centering
        \includegraphics[width=\textwidth]{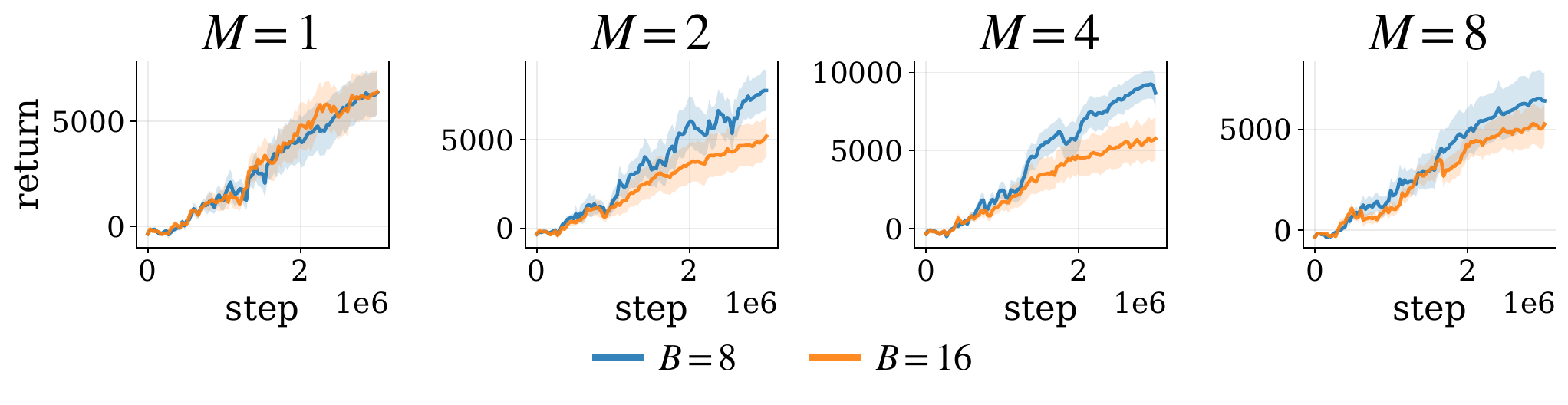}
        \caption{halfcheetah}
        \label{fig:return_bs_ablation_halfcheetah}
    \end{subfigure}
    \hfill
    \begin{subfigure}[t]{0.49\textwidth}
        \centering
        \includegraphics[width=\textwidth]{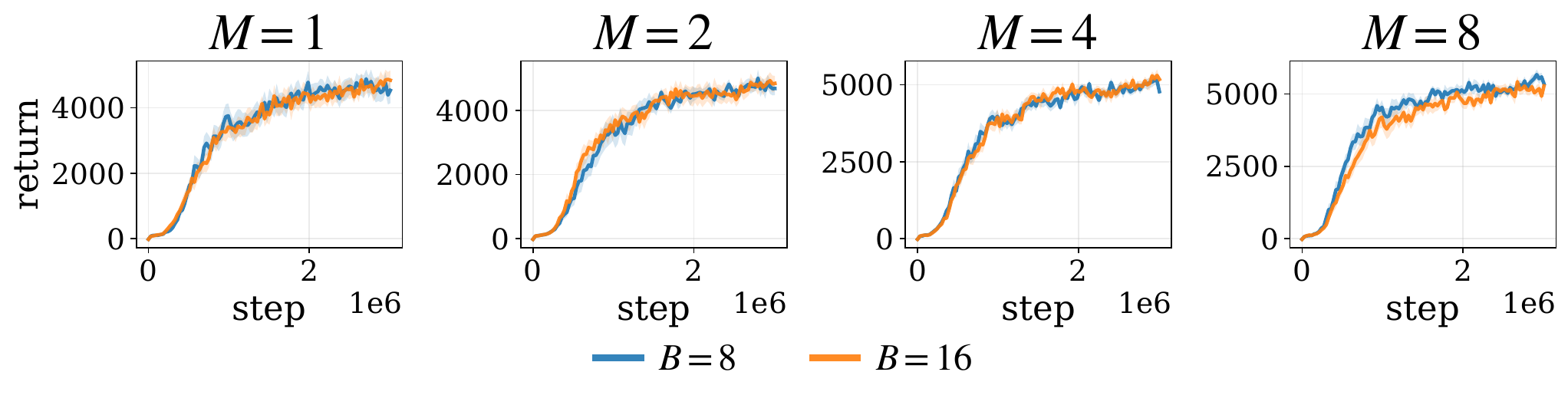}
        \caption{ant}
        \label{fig:return_bs_ablation_ant}
    \end{subfigure}

    \begin{subfigure}[t]{0.49\textwidth}
        \centering
        \includegraphics[width=\textwidth]{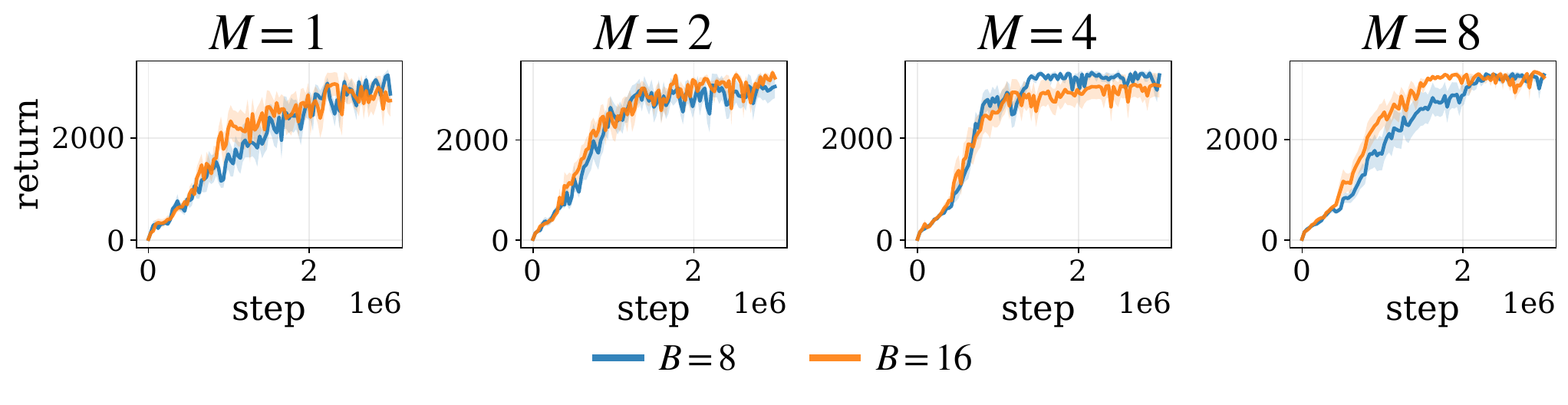}
        \caption{hopper}
        \label{fig:return_bs_ablation_hopper}
    \end{subfigure}
    \hfill
    \begin{subfigure}[t]{0.49\textwidth}
        \centering
        \includegraphics[width=\textwidth]{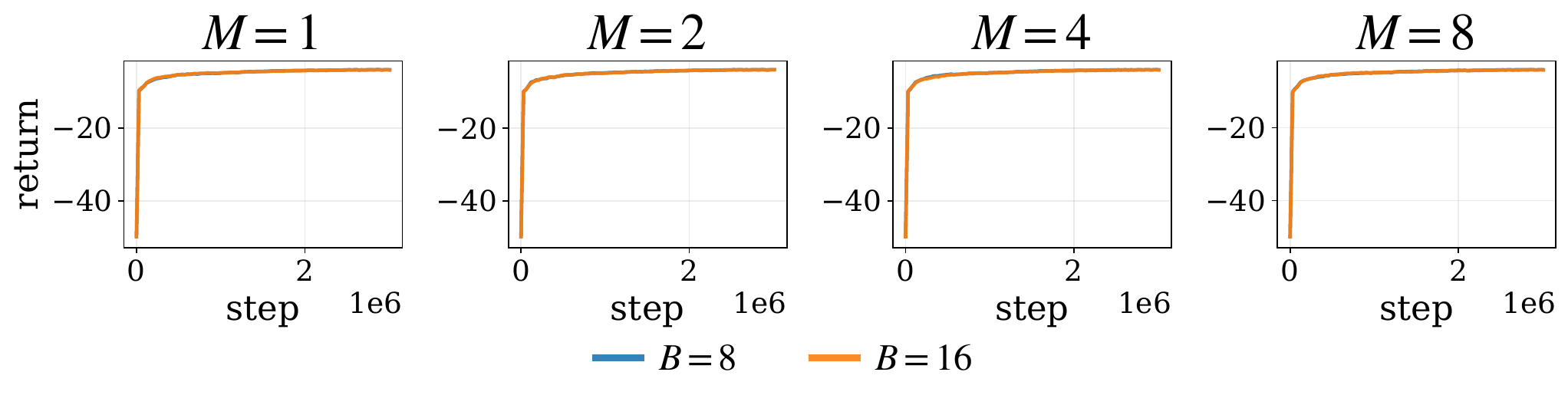}
        \caption{reacher}
        \label{fig:return_bs_ablation_reacher}
    \end{subfigure}

    \begin{subfigure}[t]{0.49\textwidth}
        \centering
        \includegraphics[width=\textwidth]{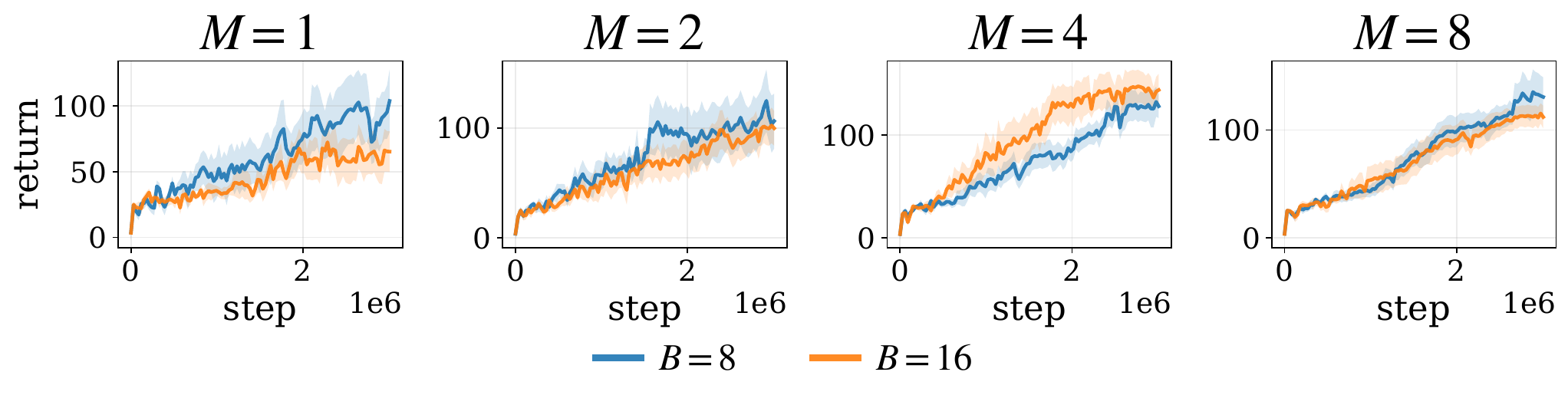}
        \caption{swimmer}
        \label{fig:return_bs_ablation_swimmer}
    \end{subfigure}
    \hfill
    \begin{subfigure}[t]{0.49\textwidth}
        \centering
        \includegraphics[width=\textwidth]{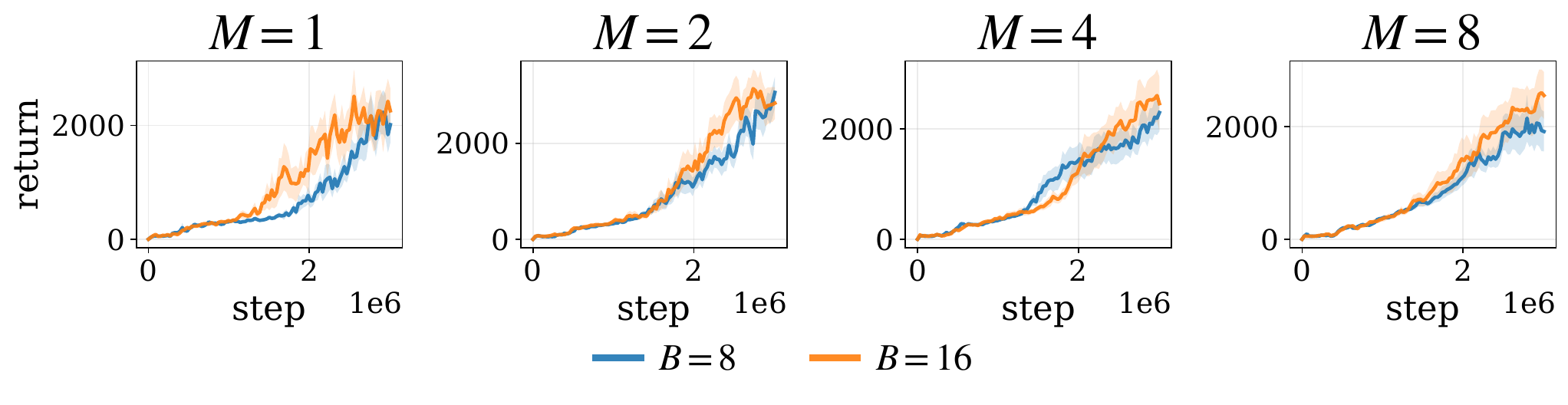}
        \caption{walker2d}
        \label{fig:return_bs_ablation_walker2d}
    \end{subfigure}
    \caption{
        The average return of ReMAC with $M=1, 2, 4, 8$ and $\epsilon=10^{-8}$ under different batch sizes $B\in\{8, 16\}$.
    }
    \label{fig:return_bs_ablation}
\end{figure}

\clearpage

\section{Verification of the analyzed mechanisms on Brax}\label{sec:app:experiments:mechanism}
In this section, we examine in more detail, on Brax, the two effects that Sec.~\ref{sec:grad_analysis} predicts: gradient damping and the scale increase.
We measure both during training, in the setting of Sec.~\ref{sec:experiments}.

\subsection{Gradient damping}
On HalfCheetah, Ant, and Hopper, Fig.~\ref{fig:rebuttal_damping} plots the actor gradient norm and the realized update norm under the Adam actor (with the default $\epsilon=10^{-8}$) and the SGD actor of App.~\ref{sec:app:experiments:sgd}, and Tab.~\ref{tab:damping} reports their ratios between $M=4$ and $M=1$.
The gradient norm decreases with $M$ under both optimizers, consistent with the $M$-decreasing upper envelope of Corollary~\ref{cor:damping_explicit}.
The update norm follows it under SGD but not under Adam, whose per-coordinate normalization by $\sqrt{\hat v_t}$ divides out the scale of the gradient (Sec.~\ref{sec:grad_analysis:adam_effect}).

\begin{figure}[h]
    \centering
    \includegraphics[width=0.9\textwidth]{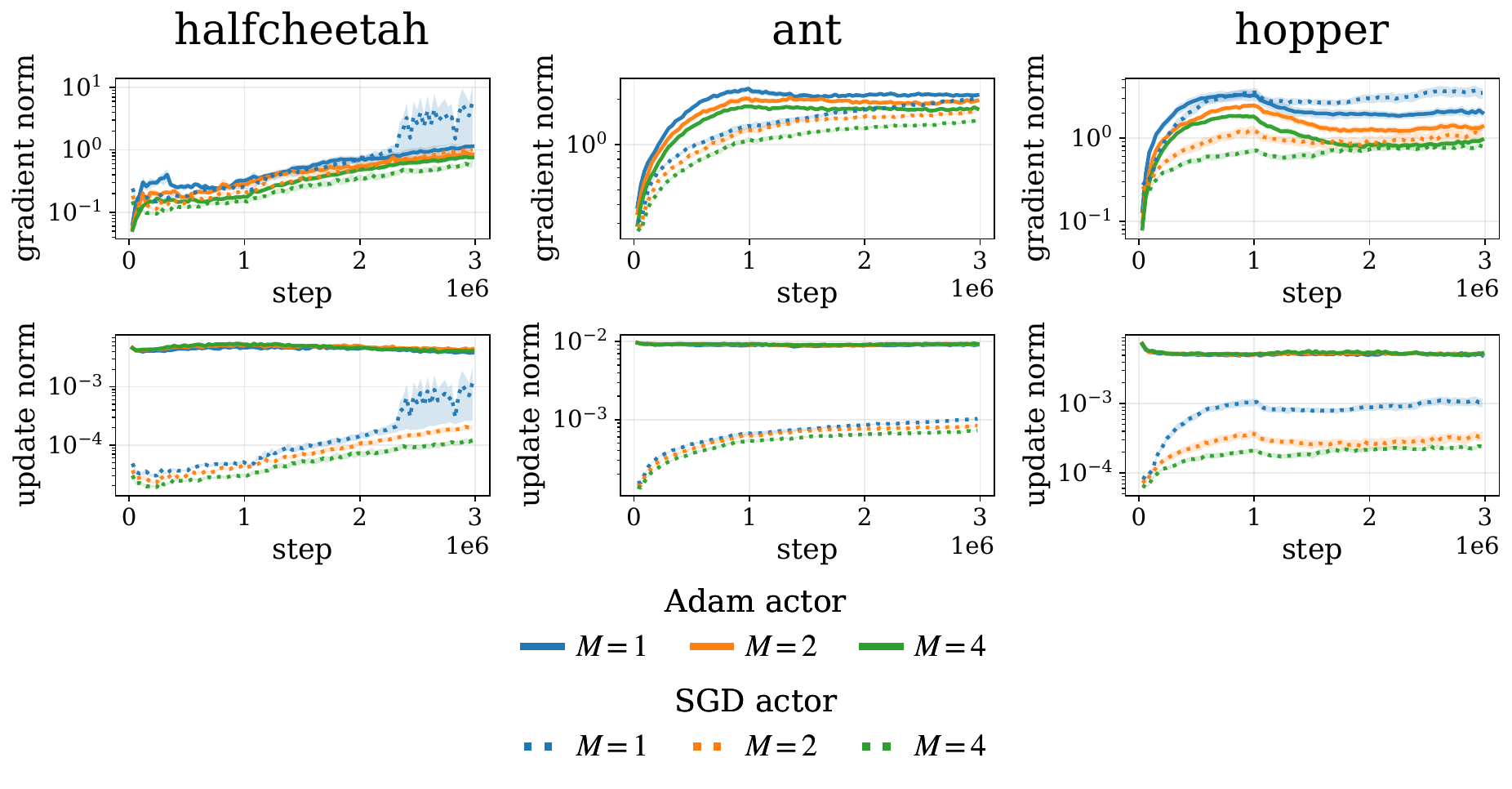}
    \caption{
        \textbf{Gradient damping on Brax.}
        Top: the actor gradient norm.
        Bottom: the realized update norm.
        Solid lines use the Adam actor, dotted lines the SGD actor; shading is the standard error over 10 seeds.
        The gradient norm shrinks with $M$ in every environment and for both optimizers, consistent with the damping mechanism suggested by Proposition~\ref{prop:gradient_damping}, Corollary~\ref{cor:damping_explicit}, and the toy analysis.
        The update norm follows the gradient norm under SGD, where it is the gradient norm scaled by the learning rate $\alpha$, but not under Adam, whose per-coordinate normalization divides out the scale of the gradient---the effect analyzed in Sec.~\ref{sec:grad_analysis:adam_effect}.
        Runs for different $M$ are trained separately, so this comparison is not a controlled isolation of the retry weighting.
    }
    \label{fig:rebuttal_damping}
\end{figure}

\begin{table}[h]
    \centering
    \caption{Gradient damping on Brax.
    Ratio of the actor gradient norm $\|\nabla_\theta L^M_{\text{actor}}\|$, and of the realized update norm $\|\Delta\theta\|$, at $M=4$ relative to $M=1$ (final values, $n=10$ seeds).
    A ratio below $1$ means the quantity is smaller at $M=4$ than at $M=1$.
    The gradient norm is reduced under both optimizers; the update norm is reduced under SGD, where it is the gradient norm times the same learning rate for every $M$, so the two SGD rows coincide, but not under Adam, whose normalization divides out the scale of the gradient.
    }
    \label{tab:damping}
    \begin{tabular}{llccc}
        \hline
        \textbf{Actor optimizer} & & \textbf{HalfCheetah} & \textbf{Ant} & \textbf{Hopper} \\
        \hline
        Adam & $\|\nabla_\theta L^M_{\text{actor}}\|$ & $0.68$ & $0.81$ & $0.44$ \\
        Adam & $\|\Delta\theta\|$ & $1.06$ & $1.03$ & $1.01$ \\
        \hline
        SGD & $\|\nabla_\theta L^M_{\text{actor}}\|$ & $0.15$ & $0.71$ & $0.22$ \\
        SGD & $\|\Delta\theta\|$ & $0.15$ & $0.71$ & $0.22$ \\
        \hline
    \end{tabular}
\end{table}

\FloatBarrier
\subsection{The scale increase}\label{sec:app:experiments:entropy_def}
The policy is a $\tanh$-squashed Gaussian (Sec.~\ref{sec:algorithm}), and the entropy we report is that of the squashed distribution,
\[
    \cH\bigl(\pi_\theta(\cdot\mid s)\bigr)
    = \sum_{i=1}^{d}\tfrac12\log\bigl(2\pi e\,\sigma_\theta(s)_i^2\bigr)
    + \E[\tilde a]{\textstyle\sum_{i=1}^{d}\log\bigl(1-\tanh^2\tilde a_i\bigr)},
\]
the entropy of the pre-squash Gaussian plus the expected log-determinant of the $\tanh$ Jacobian.
The second term is non-positive and depends on the policy mean as well as on the scale, so the entropy is not a function of the scale alone, unlike in the unbounded setting of Sec.~\ref{sec:grad_analysis}.
We report this quantity because it measures the diversity of the actions that are actually executed, which is the direct indication of the stochastic exploration.
We call the first term the Gaussian term, which depends only on the scale, and the second the $\tanh$ term, which decreases as the actions saturate at the bounds of the action space.
Fig.~\ref{fig:entropy_terms} plots both terms on HalfCheetah, Hopper, Ant, and HumanoidStandup.
In all four environments, the Gaussian term is far larger for every $M>1$ than for $M=1$: $M>1$ increases the policy scale, the direction of the mechanism in Proposition~\ref{prop:entropy_increase}.
The $\tanh$ term, in contrast, depends on the policy mean and therefore on the environment, and in some environments it decreases as training proceeds.
This saturation is most pronounced in HumanoidStandup, the highest-dimensional task, where for $M\ge4$ the decrease of the $\tanh$ term offsets the increase of the Gaussian term, so that the entropy stays close to that of $M=1$.

\begin{figure}[h]
    \centering
    \includegraphics[width=1.0\textwidth]{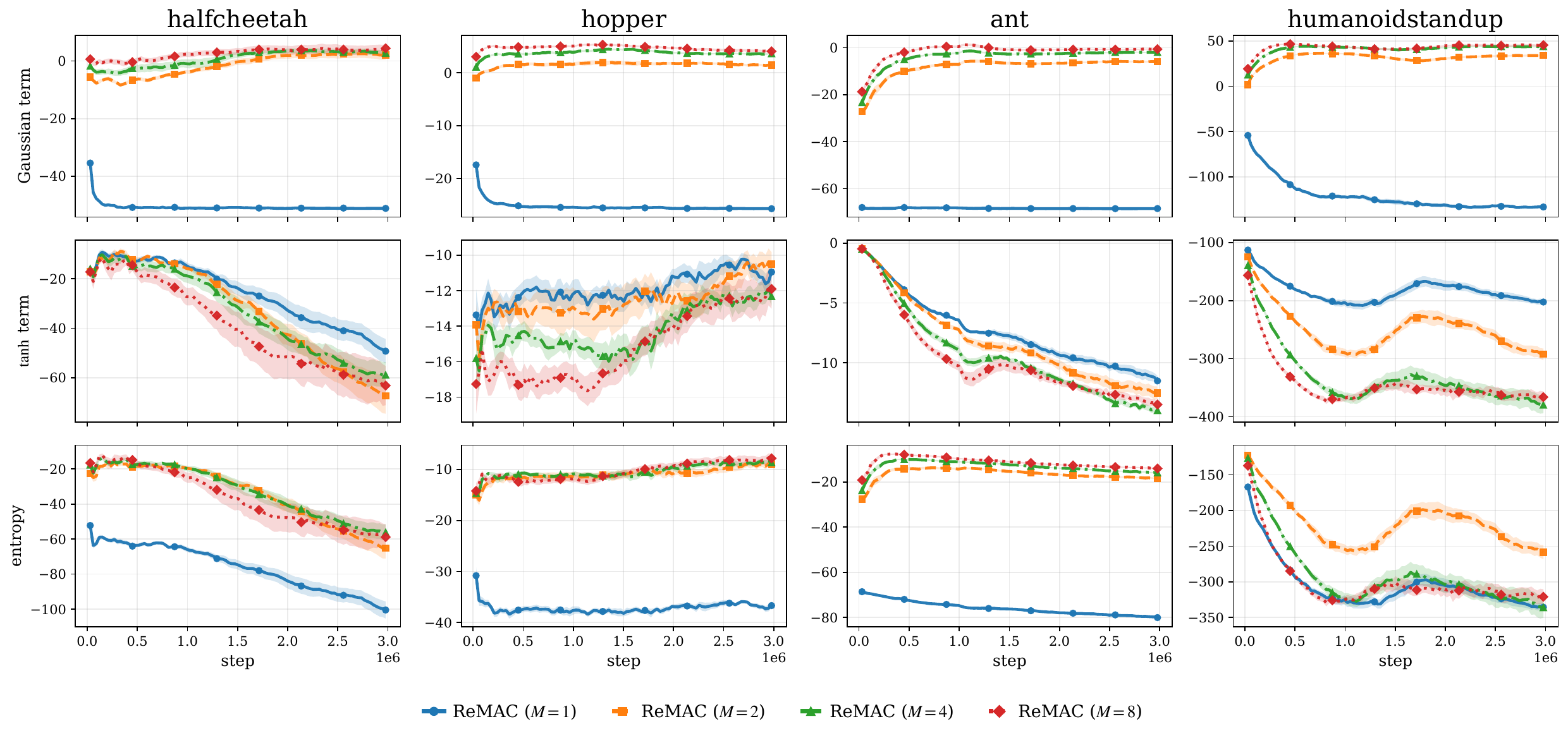}
    \caption{The two terms of the policy entropy for ReMAC with $M\in\{1,2,4,8\}$, in the setting of Figs.~\ref{fig:main_result_return} and~\ref{fig:main_entropy}.
    Top: the Gaussian term, which depends only on the scale.
    Middle: the $\tanh$ term, which decreases as the actions saturate.
    Bottom: their sum, the policy entropy.
    The shaded area shows the standard error across 10 random seeds.}
    \label{fig:entropy_terms}
\end{figure}

\clearpage

\section{Further experiments}\label{sec:app:experiments:further}
This section reports further experiments beyond the main evaluation: ReMAC with an SGD actor, the entropy of the visited-state distribution, and a sparse-reward task.

\subsection{Does ReMAC work with an SGD actor?}\label{sec:app:experiments:sgd}
Since Adam's normalization interacts with the damping effect, it is natural to ask whether ReMAC depends on Adam.
We reran ReMAC with a plain-SGD actor optimizer, leaving the critic optimizer, environment, learning rate, $M$, $B$, seeds, and step budget unchanged, so that the actor optimizer is the only difference.
ReMAC does train with an SGD actor; Tab.~\ref{tab:sgd} reports the final returns as a reference.

\begin{table}[h]
    \centering
    \caption{ReMAC with an Adam actor and with a plain-SGD actor, both with $B=8$; everything else, including the learning rate, is held fixed.
    Final return, mean $\pm$ standard error over $n=10$ seeds.
    We observe no clear separation between the two optimizers on Ant and no consistent ordering on HalfCheetah; SGD performs worse on Hopper.}
    \label{tab:sgd}
    \begin{tabular}{llcc}
        \hline
        \textbf{Environment} & $M$ & \textbf{Adam actor} & \textbf{SGD actor} \\
        \hline
        HalfCheetah & $1$ & $6204 \pm 1040$ & $3364 \pm 456$ \\
         & $2$ & $7511 \pm 1104$ & $10454 \pm 1151$ \\
         & $4$ & $9058 \pm 902$ & $7539 \pm 1283$ \\
        \hline
        Ant & $1$ & $4588 \pm 319$ & $4505 \pm 208$ \\
         & $2$ & $4766 \pm 180$ & $4710 \pm 122$ \\
         & $4$ & $5009 \pm 113$ & $4850 \pm 110$ \\
        \hline
        Hopper & $1$ & $3008 \pm 55$ & $2657 \pm 298$ \\
         & $2$ & $3003 \pm 227$ & $2309 \pm 368$ \\
         & $4$ & $3177 \pm 44$ & $2638 \pm 337$ \\
        \hline
    \end{tabular}
\end{table}

\subsection{Entropy of the visited-state distribution}\label{sec:app:experiments:coverage}
The claims of this paper concern policy stochasticity, which policy entropy measures directly, and not exploration of the state space; we also probed the stronger notion as a reference.

\paragraph{Measurement.}
Following prior work on maximum state-entropy exploration \citep{mutti2021mepol,seo2021re3}, we estimate the differential entropy of the visited observation distribution using the $k$-nearest-neighbor estimator of \citet{mutti2021mepol} with $k=4$.
At every evaluation we roll out the current policy in 64 fresh environments, mask post-termination steps, and keep the same number of visited observations per run, pooling them into $N=50500$ states per seed.
To compare methods in a shared coordinate system, we standardize each environment's observations with a single reference mean and standard deviation, computed from a balanced pool that draws the same number of states from every method and seed.
We report the results in Tab.~\ref{tab:coverage}.

\paragraph{Result.}
We find no consistent improvement over SAC, and no consistent trend in $M$ (Tab.~\ref{tab:coverage}).
Whether ReMAC improves state-space coverage therefore remains open.
Settling this question requires sparse-reward or hard-exploration domains, where temporally coherent or uncertainty-directed exploration may matter more than policy stochasticity; App.~\ref{sec:app:experiments:sparse} reports one such evaluation.

\begin{table}[h]
    \centering
    \caption{Entropy of the visited-state distribution, estimated by the $k$-nearest-neighbor estimator of \citet{mutti2021mepol} with $k=4$, mean $\pm$ standard error over $n=5$ seeds.
    Values are comparable only within an environment.
    No method is best in every environment, and larger $M$ does not consistently increase the estimate.}
    \label{tab:coverage}
    \begin{tabular}{lcccc}
        \hline
        \textbf{Environment} & \textbf{SAC} & \textbf{ReMAC ($M{=}1$)} & \textbf{ReMAC ($M{=}2$)} & \textbf{ReMAC ($M{=}4$)} \\
        \hline
        Ant & $21.78 \pm 0.60$ & $24.37 \pm 0.34$ & $24.32 \pm 0.76$ & $21.57 \pm 0.69$ \\
        HalfCheetah & $10.84 \pm 0.81$ & $-11.29 \pm 6.30$ & $4.79 \pm 1.68$ & $8.43 \pm 1.02$ \\
        Hopper & $2.47 \pm 0.42$ & $-1.72 \pm 1.86$ & $0.60 \pm 1.01$ & $-3.52 \pm 3.71$ \\
        \hline
    \end{tabular}
\end{table}

\subsection{Sparse-reward Reacher}\label{sec:app:experiments:sparse}
We also evaluate ReMAC on a sparse-reward variant of Reacher as a reference.
Following the goal-conditioned convention of JaxGCRL \citep{bortkiewicz2025jaxgcrl}, which replaces the shaped distance reward with a binary goal indicator, we set $r_t = \indic{d_t < \delta}$ with $\delta = 0.01$, where $d_t$ is the distance between the fingertip and the target.
The return is therefore the number of steps of the 50-step episode spent within $\delta$ of the goal; a uniformly random policy attained an average return of $0.15$ over $2048$ episodes in our measurement.
All other settings are those of the dense Reacher.

Tab.~\ref{tab:sparse_reacher} reports the results.
ReMAC attains mean returns comparable to SAC at every $M$.
SAC, which relies only on stochastic exploration, attains a similar return, so this task is not hard enough to require deep exploration.
Substantiating deep exploration by ReMax would therefore require harder sparse-reward or hard-exploration tasks.

\begin{table}[h]
    \centering
    \caption{Return on the sparse-reward Reacher ($\delta=0.01$; the maximum is $50$, and a uniformly random policy attained $0.15$ over $2048$ episodes).
    Final values, mean $\pm$ standard error over $n=10$ seeds.}
    \label{tab:sparse_reacher}
    \begin{tabular}{lc}
        \hline
        & \textbf{Return} \\
        \hline
        SAC & $29.24 \pm 1.65$ \\
        \hline
        ReMAC ($M=1$) & $28.64 \pm 2.44$ \\
        ReMAC ($M=2$) & $32.44 \pm 1.36$ \\
        ReMAC ($M=4$) & $31.46 \pm 1.19$ \\
        ReMAC ($M=8$) & $27.93 \pm 1.73$ \\
        \hline
    \end{tabular}
\end{table}

\end{document}